\documentclass{article} 
\usepackage{iclr2027_conference,times}
\usepackage{titletoc} 
\usepackage{wrapfig}

\usepackage{amsmath,amsfonts,bm}

\def\eqref#1{equation~\ref{#1}}

\def\1{\bm{1}}

\DeclareMathAlphabet{\mathsfit}{\encodingdefault}{\sfdefault}{m}{sl}
\SetMathAlphabet{\mathsfit}{bold}{\encodingdefault}{\sfdefault}{bx}{n}

\usepackage{hyperref}
\usepackage{url}

\usepackage{booktabs}
\usepackage{graphicx}
\usepackage{amssymb,amsthm,mathrsfs}
\usepackage{xcolor}
\usepackage{enumitem}
\usepackage{placeins}
\usepackage{capt-of}

\newtheorem{lemma}{Lemma}
\newtheorem{proposition}{Proposition}

\theoremstyle{definition}

\newcommand{\revEqref}[1]{(\ref{#1})}

\title{ASPIRE: Saddle-Point Discovery through Set Prediction and Physical Refinement}

\author{Yucheng Zhao$^{1}$, Quanyou Zhang$^{1}$, Shaoxiang Qin$^{1,2}$,
Haixuan Xu$^{1}$, Xiongye Xiao$^{1,*}$\\
$^{1}$University of Tennessee, Knoxville\quad 
$^{2}$McGill University
}

\iclrfinalcopy 
\begin{document}

\maketitle\begingroup\renewcommand{\thefootnote}{\fnsymbol{footnote}}\footnotetext[1]{Corresponding author: xxiao9@utk.edu}\endgroup
\fancyhead{}
\fancyhead[C]{Preprint}
\begin{abstract}

Predicting thermally activated diffusion and defect evolution with event-driven models requires identifying atomic rearrangement mechanisms and their activation barriers. Discovering the associated saddle points is a major computational bottleneck: multiple rearrangements may originate from one metastable state, while costly local searches can fail or repeatedly converge to the same saddle. To address this challenge, we introduce \emph{ASPIRE} (\textbf{A}tomistic \textbf{S}addle-\textbf{P}oint \textbf{I}nference with \textbf{R}efinement for \textbf{E}vents), a framework that predicts a set of saddle candidates from a single initial atomic environment and refines them through Dimer searches on the original interatomic potential. The framework’s equivariant set predictor, Ev-Quiformer, integrates (i) geometry-conditioned scalar/vector event slots for generating multiple saddle-point proposals and (ii) a decoder that maps each slot to a full atomic displacement field by combining atom, slot, and anchor-relative vectors with invariant coefficients. We also contribute two datasets: (i) \textsc{BccFe-4vacAV-4000}, comprising 4,000 four-vacancy body-centered cubic iron configurations and 65,450 reference events grouped by initial state for set supervision and post-refinement evaluation; and (ii) \textsc{BccFe-1to4Vac}, comprising 5,372 configurations with one to four vacancies each. Theoretically, we establish conditions for proposal equivariance. Experimentally, ASPIRE achieves 77.20\% reference-event coverage on this benchmark, compared with 75.73\% for a conventional Dimer baseline, while requiring approximately half as many Dimer force evaluations. In a timing evaluation on 50 configurations, ASPIRE reduces wall time per configuration from 478.8 s to 176.3 s under the stated hardware settings. The dataset is available at \href{https://huggingface.co/datasets/cheesewawa98/BccFe4VacAV-4000}{https://huggingface.co/datasets/cheesewawa98/BccFe4VacAV-4000}.

\end{abstract}


\section{Introduction}
\label{sec:revised-intro}

\begin{figure}[t]
\centering
\includegraphics[width=\linewidth]{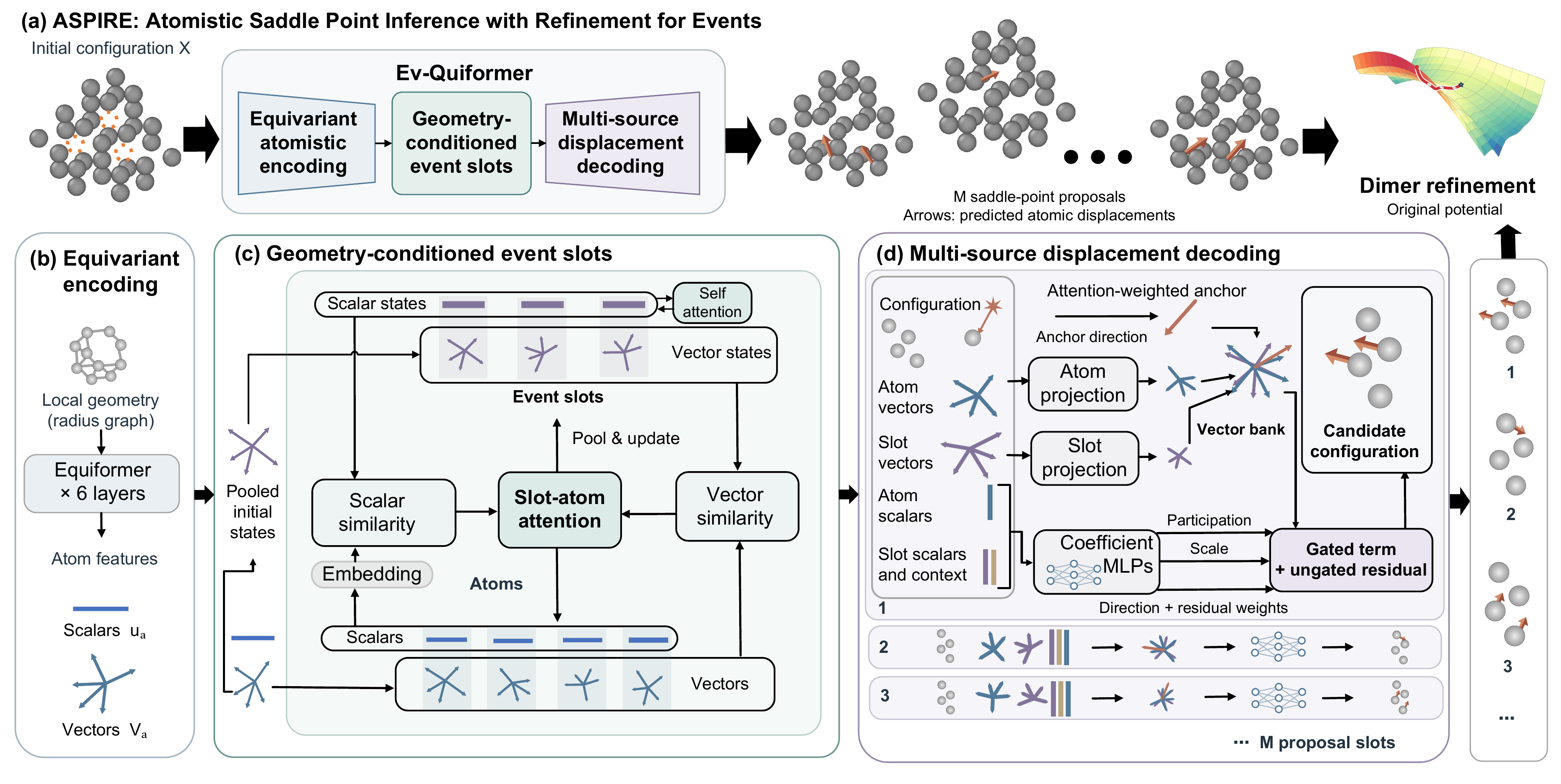}
\caption{\textbf{ASPIRE: Atomistic Saddle-Point Inference with Refinement for Events.}
\textbf{(a) Overall workflow.} Ev-Quiformer generates $M$ saddle-point proposals from a single initial configuration, followed by physical refinement using Dimer searches on the original interatomic potential. Panels~(b--d) detail its three neural modules.
\textbf{(b) Equivariant encoding.} A six-layer Equiformer extracts invariant scalar and equivariant vector features for each atom.
\textbf{(c) Geometry-conditioned event slots.} Scalar self-attention and repeated slot--atom attention update slots initialized from learned scalar seeds and pooled atom vectors.
\textbf{(d) Multi-source displacement decoding.} Invariant coefficients combine atom, slot, and anchor-relative vectors into a full atomic displacement field through a gated term and an ungated residual. Adding this field to the initial configuration yields a candidate for subsequent physical refinement.}
\label{fig:saddleset-overview}
\vspace{-5mm}
\end{figure}

{


Nuclear fission reactors and future fusion systems require structural materials that retain their mechanical integrity and dimensional stability under irradiation and elevated temperatures~\citep{du2023superior}. In iron-based alloys, irradiation-induced clusters and dislocation loops contribute to hardening~\citep{shimodaira2018contribution}, while the growth of cavities causes swelling~\citep{lin2021temperature}. Understanding these changes requires resolving the individual and collective atomic rearrangements that drive defect migration, recombination, and interactions. For example, two interstitial loops in irradiated iron can interact to form a new loop structure, a mechanism revealed by atomistic simulation~\citep{PhysRevLett.110.265503}. Vacancy diffusion also enables dislocation climb during creep, with migration barriers influenced by the local dislocation environment~\citep{kabir2010predicting}. 


Molecular dynamics (MD) follows atomic motion directly, but sampling rare thermally activated events requires resolving fast atomic vibrations over long waiting times. Event-based methods reach longer time scales by identifying the possible rearrangements and using their rates to select events and advance time \citep{voter2007introduction}. On-the-fly methods, including adaptive KMC \citep{henkelman2001long}, SEAKMC \citep{xu2012self,XU2015135}, and k-ART \citep{beland2011kinetic}, search for these events as the simulation proceeds. In event-based simulations, saddle discovery identifies transition mechanisms and activation barriers used to estimate event rates. Saddle discovery in this setting has several challenges: (i) an initial configuration generally admits multiple possible events whose number and product configurations are unknown in advance; (ii) the recovered event catalogue affects event selection and estimated waiting times; (iii) and new configurations encountered along a trajectory may require renewed searches, making event discovery computationally costly~\citep{HAYAKAWA2021110785,LIANG2023112354}.

 
}


Existing learning-based saddle (transition-state) discovery methods differ in their required inputs and target applications. React-OT predicts transition states from reactant and product
configurations \citep{duan2025optimal}, and OA-ReactDiff is primarily
evaluated in the same double-ended setting \citep{duan2023accurate}.
Conditioning on both endpoints does not directly address on-the-fly saddle discovery when
products are unknown. ReactionAtlas instead proposes transition states
from reactants and uses kinetic modeling to expand reaction networks
\citep{gugler2026reactionatlas}. 

Our focus is efficient saddle-point prediction for multi-vacancy configurations in body-centered cubic (bcc) iron (Fe), a model system for studying defect evolution in iron-based materials. We use a fixed embedded-atom method (EAM) potential for both reference-event generation and candidate refinement. We aim to identify multiple distinct events from a single initial configuration, without supplied product states, while reducing physical search cost. A propose-refine paradigm uses learned models to generate candidates for numerical refinement, aiming to shorten the search while retaining the underlying physical model. Variants of this paradigm have been used in a range of scientific applications, including electronic-structure calculations
\citep{schutt2019unifying}, protein structure prediction
\citep{jumper2021highly}, and materials discovery
\citep{merchant2023scaling}.


\textbf{Our contributions.} We formulate saddle discovery in interacting defect systems as a set-learning problem, where multiple candidate rearrangements are jointly predicted from a single initial configuration without supplied product states. We introduce \emph{ASPIRE} (Atomistic Saddle-Point Inference with Refinement for Events; Figure~\ref{fig:saddleset-overview}), which combines geometric prediction with physical refinement, and establish dedicated datasets and a benchmark for evaluating distinct-event recovery and physical search cost. Our contributions are threefold:

\noindent\textbf{(i) Problem formulation and datasets.} We make multiple rearrangements from the same initial state a joint learning target, without requiring product configurations. Our bcc Fe benchmark groups 65,450 reference events by 4,000 four-vacancy configurations, enabling set supervision and evaluation of distinct event recovery after refinement.

\noindent\textbf{(ii) Framework design and theory.}
We propose ASPIRE, combining learned saddle proposals with physical
refinement. Its neural predictor, Ev-Quiformer, integrates
geometry-conditioned scalar/vector event slots and a three-source
decoder to generate full atomic displacement fields, supervised by
sequential matching. We establish conditions for proposal equivariance
and relate geometric error to reference-event recovery under
local-refinement assumptions.

\noindent\textbf{(iii) Empirical evaluation.}
We compare ASPIRE with conventional Dimer and event-reuse baselines,
alongside ablation studies. On 200 four-vacancy test configurations,
ASPIRE achieves 77.20\% reference-event coverage versus Dimer's 75.73\%,
with roughly half the recorded search-force calls under the respective
search protocols.


\section{Related Work}

\paragraph{Rare-event simulation, saddle search, and reuse.}
Kinetic Monte Carlo (KMC) simulates transitions between metastable states
by sampling events and waiting times  from their rates (see Appendix~\ref{app:rev-kmc-background}) \citep{voter2007introduction}.
On-the-fly KMC discovers events through saddle searches during simulation,
rather than requiring a predefined event catalogue. SEAKMC localizes
these searches to active volumes around defects \citep{xu2012self,XU2015135}. The Dimer method \citep{henkelman1999dimer} finds first-order saddle points without a supplied
product configuration (see Appendix~\ref{app:rev-dimer-background}). The kinetic database (KDB) matches local atomic geometries to stored
mechanisms and supplies saddle-search initializations \citep{terrell2012database}.
k-ART combines on-the-fly discovery with an event catalogue indexed by
local topology for event reuse
\citep{beland2011kinetic}. 


\paragraph{Equivariant atomistic models.}
Tensor Field Networks, SE(3)-Transformers, EGNN, and Equiformer encode
geometric transformation rules
\citep{thomas2018tensor,fuchs2020se,satorras2021n,liao2022equiformer}.
NequIP and MACE learn interatomic energies and forces
\citep{batzner20223,batatia2022mace}.

\paragraph{Set prediction.}
Set Transformer uses learnable seeds for attention-based pooling and
self-attention among pooled outputs \citep{lee2019set}.
DETR combines learned queries, query self-attention, cross-attention, and
bipartite matching for unordered object prediction \citep{carion2020end}.
Slot Attention iteratively refines slots
\citep{locatello2020object}.
Geometric query designs include DAB-DETR's dynamic anchor boxes
\citep{liu2022dab} and Invariant Slot Attention's attention-derived
slot-centric reference frames and relative-coordinate decoding
\citep{biza2023invariant}.
For one-to-many supervision, Group DETR performs matching within query
groups and MS-DETR supervises multiple
queries in the primary decoder
\citep{chen2023group,10655699}. We adapt these set-prediction designs with
equivariant atomistic representations to build Ev-Quiformer for jointly predicting an unordered set
of candidate saddle configurations from a single initial configuration.

\section{Problem Formulation}
\label{sec:revised-problem}
Given an initial metastable configuration, our goal is to predict a set of saddle candidates for subsequent physical refinement, without requiring product states. Let $X\in\mathbb R^{N\times3}$ be an initial configuration at a
\emph{metastable state}, a nondegenerate local minimum of
the fixed interatomic potential
$V:\mathbb R^{N\times3}\to\mathbb R$, with
$\nabla V(X)=0$ and $\nabla^2V(X)\succ0$. Here $\nabla V$ and $\nabla^2V$ are taken in the fixed
free-coordinate space, with the surrounding environment and
constraints already imposed
(Appendix~\ref{app:rev-free-coordinates}).
We consider the \emph{open saddle configuration set prediction problem},
in which the number, identities, and outcomes of accessible events
are unknown a priori.

A predictor outputs an unordered collection of displacement fields, denoted by
\begin{equation}
 F_{\theta}(X)=\{\widehat D_m\}_{m=1}^{M},
 \qquad
 \widehat D_m\in\mathbb R^{N\times3}.
\end{equation}
The corresponding proposal configurations are
$\{\widehat S_m\}_{m=1}^{M}
=\{X+\widehat D_m\}_{m=1}^{M}$.
After refinement, the set of retained saddle configurations is defined as
\begin{equation}
 \{S_k\}_{k=1}^{K}
 =
 \{\mathcal R_X(\widehat S_m)\}_{m=1}^{M}
 \setminus\{\bot\},
 \qquad K\le M,
 \label{eq:rev-refined-set}
\end{equation}
where each $S_k$ is a configuration at a \emph{transition state}
satisfying:
(i) $S_k$ is a first-order saddle on $V$, that is,
$\nabla V(S_k)=0$ and
$\lambda_1\!\bigl(\nabla^2V(S_k)\bigr)
<0<
\lambda_2\!\bigl(\nabla^2V(S_k)\bigr)$,
where $\lambda_1\le\lambda_2\le\cdots$ are the Hessian eigenvalues
and the unit eigenvector $v_1(S_k)$ associated with $\lambda_1$
is its unstable mode;
(ii) relaxing from $S_k\pm\epsilon v_1(S_k)$ reaches $X$
and another distinct product configuration $X_k^+$ at a
metastable local minimum.
The map
$\mathcal R_X:\mathbb R^{N\times3}
\to\mathbb R^{N\times3}\cup\{\bot\}$
is an idealized search-and-validation map: it refines each proposed
configuration on $V$ and returns $S_k$ only when the target saddle
and connectivity conditions above hold, or $\bot$ otherwise.

We define the resulting event set as
\begin{equation}
 \mathcal E_X
 =\{e_k\}_{k=1}^{K}
 =\bigl\{(X,S_k,X_k^+,\Delta E_k)\bigr\}_{k=1}^{K},
 \qquad
 \Delta E_k=V(S_k)-V(X)>0.
 \label{eq:rev-event}
\end{equation}
The resulting event set is determined by $X$ and the validated
saddle configurations. The objective is to recover as many
distinct events as possible at low search cost, without using
any product configuration $X_k^+$ as input. Throughout this paper, we use Dimer-based search and validation
as the numerical instantiation of this framework
(see Appendix~\ref{app:rev-dimer-background}).

\section{Methods}
\label{sec:revised-method}

\paragraph{Intuition} Several atomic rearrangements may originate from the same metastable
configuration, each associated with a saddle. We therefore predict a set of
approximate saddle geometries and use them to initialize Dimer refinement on
the original interatomic potential. Learning proposes alternative rearrangements,
while force-based refinement seeks nearby saddle points. Proposition~\ref{prop:rev-ideal-radius} establishes local convergence of the
ideal minimum-mode iteration from a sufficiently close initialization~\citep{luo2022convergence}.

Based on this intuition, we develop \emph{ASPIRE}
(Figure~\ref{fig:saddleset-overview}), integrating set prediction with physical
refinement. We first describe its neural predictor, \emph{Ev-Quiformer},
through equivariant environment encoding (Section~\ref{sec:rev-encoder}),
geometry-conditioned event slots (Section~\ref{sec:rev-slots}), and
three-source displacement decoding (Section~\ref{sec:rev-decoder}). We then
present the set-matching training objective (Section~\ref{sec:rev-loss}) and
the optional confidence readout for selecting proposals for Dimer refinement
(Section~\ref{sec:rev-judge}). The conditions and proof for proposal equivariance
are given in Appendix~\ref{app:rev-equivariance}.


\subsection{Equivariant attention environment encoding}
\label{sec:rev-encoder}
The encoder reads the initial configuration $X$. A directed radius graph connects atoms $a,b$ within a maximum radius distance, with a limit to the number of neighbors per atom and edge vectors $X_a-X_b$, where $X_a\in\mathbb R^{3}$ is the $a$-th row of $X$. Six Equiformer blocks \citep{liao2022equiformer} attend over these neighbors and return, for every atom, an invariant scalar feature $u_a\in\mathbb R^{128}$ and an equivariant vector feature $V_a\in\mathbb R^{64\times3}$ that rotates with the input. No product configurations $X_k^+$ are supplied. Edge features and layer sizes are given in Appendix~\ref{app:rev-encoder}.

\subsection{Geometry-conditioned event slots}
\label{sec:rev-slots}

\paragraph{Slots.}
A slot $m$ carries a scalar fingerprint initialized from a learned seed
$s_m\in\mathbb R^{d_s}$ and an equivariant vector state
$Q_m\in\mathbb R^{d_v\times3}$ constructed from the encoded atomic geometry.
Learned scalar seeds distinguish hypotheses, while geometry-pooled vector states
retain directional information without introducing fixed learned spatial axes.

\paragraph{Stages.}
The slot module proceeds through initialization, vector--scalar updates,
scalar-only updates, and final readout, using
Equations~\revEqref{eq:rev-slot-self}--\revEqref{eq:rev-slot-mha-update}
with the stage types summarized in Table~\ref{tab:rev-slot-stages}.
At stage $t$, let $\hat s_m$ and $\hat u_a$ denote the corresponding
normalized slot and atomic scalar features, and let $H^{(t)}$ be the number
of attention heads with head width $d_h$.
The exact stage schedule, normalizations, initialization, dimensions, and
parameter-sharing choices are given in Appendix~\ref{app:rev-slot-details}.

\paragraph{Slot self attention.}
Fingerprints optionally self-attend across slots, in place,
\begin{equation}
 s_m^{(t)}
 \leftarrow
 (1-\eta^{(t)})s_m^{(t)}
 +\eta^{(t)}\,
 \operatorname{SA}^{(t)}
 \!\bigl(s_1^{(t)},\dots,s_M^{(t)}\bigr)_m,
 \qquad
 \eta^{(t)}\in\{0,1\}
 \label{eq:rev-slot-self}
\end{equation}

\paragraph{Slot--atom attention.}
At stage $t$, attention head $h$ uses scalar similarity and, when enabled,
vector alignment to compute the slot--atom attention weights:
\begin{equation}
 A^{(t,h)}_{ma}
 =
 \operatorname{softmax}_a\!\left[
 \frac{
 (W_q^{(t,h)}\hat s_m)^\top
 (W_k^{(t,h)}\hat u_a)
 }{\sqrt{d_h}}
 +
 \gamma^{(t,h)}
 \frac{
 \left\langle
 \widetilde W_Q^{(t,h)}Q_m^{(t)},
 \widetilde W_V^{(t,h)}V_a
 \right\rangle_F
 }{\sqrt{3d_v}}
 \right].
 \label{eq:rev-slot-attn}
\end{equation}

Each head pools a scalar and a vector context; scalar contexts are concatenated
over heads,
\begin{equation}
\begin{gathered}
 z_m^{(t,h)}
 =
 \sum_a A^{(t,h)}_{ma}\,
 W_z^{(t,h)}\hat u_a,
 \qquad
 z_m^{(t)}
 =
 \bigl[
 z_m^{(t,1)};
 \dots;
 z_m^{(t,H^{(t)})}
 \bigr]
 \in\mathbb R^{d_s},
 \\
 U_m^{(t,h)}
 =
 \sum_a A^{(t,h)}_{ma}\,
 W_V^{(t,h)}V_a
 \in\mathbb R^{d_v\times3}.
 \label{eq:rev-slot-context}
\end{gathered}
\end{equation}

The scalar maps act on invariant features, while vector maps mix channels
identically across spatial components. Their exact parameterization and
sharing pattern are specified in Appendix~\ref{app:rev-slot-details}.

\paragraph{Update.}
Using the pooled scalar and vector contexts, we update the slot states according to
\begin{equation}
 s_m^{(t+1)}
 =
 s_m^{(t)}
 +
 F^{(t)}
 \!\Bigl(
 s_m^{(t)},
 z_m^{(t)},
 \bigl(\|U^{(t,h)}_{m,v}\|_2\bigr)_{h,v}
 \Bigr),
 \qquad
 Q_m^{(t+1)}
 =
 Q_m^{(t)}
 +
 \sum_h W_U^{(t,h)}U_m^{(t,h)} .
 \label{eq:rev-slot-update}
\end{equation}

$F^{(t)}$ is zero, a two-layer MLP on its three arguments, or, in the
multi-head stages,
\begin{equation}
 F^{(t)}(s,z,\cdot)
 =
 W_o^{(t)}z
 +
 \operatorname{FF}^{(t)}
 \!\bigl(
 \operatorname{LN}^{(t)}
 (s+W_o^{(t)}z)
 \bigr),
 \label{eq:rev-slot-mha-update}
\end{equation}
which ignores the vector norms. The final readout returns
$A^{(T)}_{ma}$, $z_m^{(T)}$, $s_m^{(T)}$, and $Q_m^{(T)}$ for the
decoder in Section~\ref{sec:rev-decoder}.

\begin{table}[h]
\caption{\textbf{Slot stage types.}
Settings of
Equations~\revEqref{eq:rev-slot-self}--\revEqref{eq:rev-slot-mha-update}.
``MLP'' is a two-layer perceptron on $[s,z,\|U\|]$;
``$W_oz$+FF'' is Equation~\revEqref{eq:rev-slot-mha-update};
``--'' means the map is absent;
$\gamma$ of the vector--scalar stage is learned from initial value 1.
Detailed stage counts and parameterization are given in
Appendix~\ref{app:rev-slot-details}.}
\label{tab:rev-slot-stages}
\centering
\small
\setlength{\tabcolsep}{3.5pt}
\begin{tabular}{@{}lccccccl@{}}
\toprule
& \multicolumn{3}{c}{Scalar attn.}
& Vector attn.
& \multicolumn{3}{c}{Update} \\
\cmidrule(lr){2-4}
\cmidrule(lr){5-5}
\cmidrule(lr){6-8}
Stage type
& $H$
& $\eta$
& $\mathrm N$
& $\gamma$
& $F$
& $W_U$
& Updated \\
\midrule
Initialization
& 1
& 0 or 1
& Id
& $0$
& $0$
& $I$
& $s_m^{(t+1)}=s_m^{(t)}$, $Q_m^{(t+1)}=U_m^{(t)}$ \\
Vector--scalar
& 1
& 0
& Id
& learned
& MLP
& learned
& $s_m^{(t+1)}$, $Q_m^{(t+1)}$ \\
Scalar-only
& 8
& 0
& LN
& $0$
& $W_oz$+FF
& --
& $s_m^{(t+1)}$, $Q_m^{(t+1)}=Q_m^{(t)}$ \\
Readout
& 1
& 0
& Id
& $1$
& --
& --
& none; outputs $A^{(T)}_{ma},z^{(T)}_m,s^{(T)}_m,Q^{(T)}_m$ \\
\bottomrule
\end{tabular}
\end{table}

\subsection{Displacement decoding from three vector sources}
\label{sec:rev-decoder}

Using the final readout outputs, with stage indices omitted, we form a
bank of 32 atom-derived vectors, eight slot-derived vectors, and one
direction from the anchor $b_m=\sum_aA_{ma}X_a$ to atom $a$:
\begin{equation}
\begin{gathered}
 \mathcal B_{ma}=
 \bigl[\,W_{\mathrm{atom}}V_a;\ W_{\mathrm{slot}}Q_m;\
 \operatorname{unit}_\epsilon\!\bigl(X_a-b_m\bigr)\,\bigr]
 \in\mathbb R^{41\times3},\\
 \operatorname{unit}_\epsilon(v)
 =\frac{v}{\max(\|v\|_2,\epsilon)},\quad\epsilon=10^{-8},
\end{gathered}
 \label{eq:rev-bases}
\end{equation}
Here $W_{\mathrm{atom}}$ and $W_{\mathrm{slot}}$ are shared, learned,
bias-free vector-channel maps. 

\paragraph{Invariant conditioning.}
The decoder is conditioned on the invariant features
\begin{equation}
\begin{gathered}
 \xi_{ma}=\bigl[\,s_m,\ z_m,\ u'_a,\ \chi_{ma},\
 \log\bigl(1+\|X_a-b_m\|_2^2\bigr)\bigr],\\
 \chi_{ma}=\Bigl(\tfrac1{\sqrt3}
 \bigl\langle(G_QQ_m)_g,\,(G_VV_a)_g\bigr\rangle\Bigr)_{g=1}^{8},
\end{gathered}
 \label{eq:rev-decoder-cond}
\end{equation}
where $u'_a=u_a^{(T)}$ is the readout's atomic embedding, and $G_Q,G_V$
are learned vector-channel maps.

\paragraph{Coefficient networks and displacement.}
Four separately parameterized MLPs, shared across atoms and slots, give \begin{equation}
\begin{aligned}
 p_{ma}&=\sigma\bigl(g_{\mathrm{act}}(\xi_{ma})\bigr),
 &\kappa_{ma}&=\operatorname{softplus}
   \bigl(g_{\mathrm{mag}}(\xi_{ma})\bigr),\\
 q_{ma}&=f_{\mathrm{coef}}(\xi_{ma}),
 &q^{\mathrm{res}}_{ma}&=f_{\mathrm{res}}(\xi_{ma}),
\end{aligned}
 \label{eq:rev-decoder-coef}
\end{equation}
where $\sigma$ is the sigmoid and
$q_{ma},q^{\mathrm{res}}_{ma}\in\mathbb R^{41}$.
The resulting displacement is
\begin{equation}
 \widehat D_{m,a}
 =\underbrace{p_{ma}}_{\text{participation}}\;
 \underbrace{\kappa_{ma}}_{\text{scale}}\;
 \underbrace{\operatorname{unit}_\epsilon
   \!\bigl(q_{ma}^\top\mathcal B_{ma}\bigr)}_{\text{direction}}
 \;+\;\gamma_{\mathrm{res}}\,
 (q^{\mathrm{res}}_{ma})^\top\mathcal B_{ma},
 \label{eq:rev-decoder}
\end{equation}
This gives atom $a$'s displacement for proposal $m$, with learned global
scale $\gamma_{\mathrm{res}}$. The gate permits localized motion without
hard sparsity, while the ungated residual accommodates surrounding
relaxations. {For a vacancy event, this allows the proposal to describe both the atoms moving toward vacant sites and the smaller displacements of their neighbors, without selecting the moving atoms in advance.} The periodic coordinate construction is described in Appendix~\ref{app:rev-periodic-background}; network widths and initialization are given in Appendix~\ref{app:rev-decoder-details}.

\subsection{Set matching and training objective}
\label{sec:rev-loss}

Let $\mathscr E_x=\{1,\ldots,K\}$ index the reference set of $X$,
with saddle displacement fields $D_e=S_e-X$. Slots have no prescribed targets. Using the scalar SmoothL1
penalty $\rho_\beta$, define the proposal--target cost over all $N$ atoms by
\begin{equation}
 C_{me}=\ell_\beta(\widehat D_m-D_e),\qquad
 \ell_\beta(E)=\frac1N\sum_{a,d}\rho_\beta(E_{ad}).
 \label{eq:rev-match-cost}
\end{equation}
Rectangular Hungarian matching gives the minimum-cost one-to-one
first-round assignment $\Pi_1$, assigning each target to a distinct slot
when $M\geq K$. We remove the matched slots and repeat against the full
catalogue for at most $R_{\mathrm{match}}=4$ rounds in total, stopping
when no slots remain. Each slot is used at most once, whereas a target
may supervise multiple slots across rounds. For $R_{\mathrm{match}}=1$, the auxiliary term is omitted. The auxiliary term is also defined to be zero whenever no pair is matched after the first round, including when the first round exhausts all slots.


Let $\Pi_+$ denote all pairs matched after the first round. The objective is
\begin{equation}
 \mathcal L_x=\lambda_D\left[
 \frac1{|\Pi_1|}\sum_{(m,e)\in\Pi_1}C_{me}
 +\frac{\lambda_+}{|\Pi_+|}
  \sum_{(m,e)\in\Pi_+}C_{me}
 \right],
 \qquad \lambda_D=10,\quad\lambda_+=0.5.
 \label{eq:rev-train-loss}
\end{equation}

\subsection{Confidence readout for proposal selection}
\label{sec:rev-judge}

We hypothesize that features learned through displacement matching also
encode catalogue agreement. We freeze the predictor and train a
lightweight MLP \emph{judge} on invariant latent and geometric features
(Appendix~\ref{app:rev-judge}). It outputs $\widehat p_m\in(0,1)$
without force evaluations; only proposals with $\widehat p_m\geq\tau$
undergo Dimer refinement. 

For proposal $\widehat D_m$ and catalogue field $D_e$, define the
moving-union root-mean-square (RMS) error by
\begin{equation}
\begin{gathered}
 U_{me}=\{a:\max(\|\widehat D_{m,a}\|_2,\|D_{e,a}\|_2)
 \geq0.3\,\text{\AA}\},\\
 \operatorname{RMS}_U(\widehat D_m,D_e)=
 \left(\frac{1}{|U_{me}|}\sum_{a\in U_{me}}
 \|\widehat D_{m,a}-D_{e,a}\|_2^2\right)^{1/2}.
\end{gathered}
\label{eq:rev-union-rms}
\end{equation}
Empty unions have infinite error. For the nearest catalogue field
$D_{e^*}$ under this RMS, set $y_m=1$ iff the RMS is below
$0.35\,\text{\AA}$ and the maximum atomwise error over $U_{me^*}$
is below $0.5\,\text{\AA}$.
Judge selection provides a tunable trade-off between reference-event recovery and physical search cost. The judge is trained on the training set only.


\section{Experiments and Results}
\label{sec:experiments}

\begin{figure}[!t]
    \centering
    \includegraphics[width=\linewidth]
    {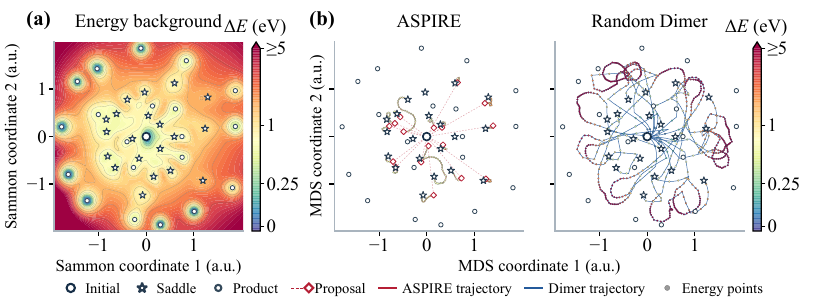}
    \caption{\textbf{Energy background, saddles, and trajectories.}
    \textbf{(a)} Sammon-projected view of the energy function and initial, saddle, and product configurations.
    \textbf{(b)} MDS-coordinate views of ASPIRE and random-start
    Dimer searches on one configuration. ASPIRE proposal jumps effectively predict approximate saddles, followed by light Dimer refinement, whereas random Dimer goes long paths;
    Each trajectory energy point denotes every five translation steps, and color denotes the energy value above initial at the point. (Appendix~\ref{app:rev-trajectory-mds}). a.u.: arbitrary units.
    }
    \label{fig:rev-recovery-cost}
\end{figure}

\subsection{Datasets}
\label{sec:rev-datasets}

\begin{wrapfigure}{r}{0.49\textwidth}\vspace{-2mm}
    \centering
    \includegraphics[width=\linewidth]{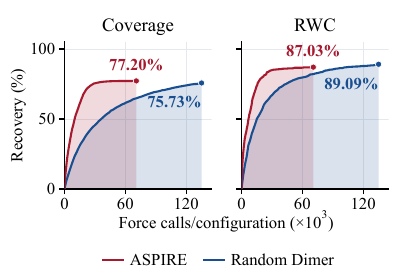}
    \caption{\textbf{Reference-event recovery versus search cost.}
    Coverage (left) and RWC (right) versus the mean recorded
    Dimer search force calls per configuration, evaluated on
    200 test configurations.
    Curves vary the judge threshold for ASPIRE
    (up to 96 proposals) or the number of random-start
    Dimer searches (up to 80).}
    \label{fig:rev-recovery-curves}
\vspace{-3mm}
\end{wrapfigure}

We construct both datasets by removing one to four lattice sites from
a 128,000-atom bcc Fe block and relaxing under the M07 EAM potential (Appendix~\ref{app:rev-fe-background}).
Each example pairs a relaxed active volume $X$ of radius
$11\,\text{\AA}$ with a reference event set
$\mathcal E_X^{\mathrm{ref}}$ collected through
80 random-start Dimer searches and SEAKMC validation, with barriers
$\Delta E_k\in[0.1,2.0]\,\mathrm{eV}$.
These reference sets are finite and
not exhaustive, but will be updated with more Dimer searches and made more complete. \textsc{BccFe-4vacAV-4000} contains 4,000 symmetry-distinct four-vacancy
configurations and 65,450 events ($2\leq K\leq43$), with a 3,700/100/200 split. It is used for all evaluations in this paper. \textsc{BccFe-1to4Vac} contains 5,372 configurations with one to four
vacancies and 83,372 events. It is built for future KMC trajectory simulation and analysis. Construction details are in
Appendix~\ref{app:rev-datasets}.

\subsection{Evaluation protocol}

Let $\mathcal X$ be the fixed evaluation set and
$\mathcal E_X^{\mathrm{rec}}\subseteq\mathcal E_X^{\mathrm{ref}}$
the recovered references at $X$. Harmonic reference rates are
$k_e=\nu e^{-\Delta E_e/(k_{\mathrm B}T)}\in(0,\infty)$ with common prefactor $\nu$. \label{ass:rev-match} We consider events $k$ and $k'$ from the same $X$ match if and only if $\max_a\|D_{k,a}-D_{k',a}\|_2\leq0.05\,\text{\AA}
 \qquad$and$\qquad
 |\Delta E_k-\Delta E_{k'}|\leq0.02\,\mathrm{eV}$\label{eq:rev-match},
 where $D_k=S_k-X$ is the saddle displacement, $\Delta E_k$ the barrier,
and $a$ indexes atoms. Figure~\ref{fig:rev-recovery-cost}(a) is a visualization of the initial, saddle, and product configurations on the projected energy surface. We use coverage $\mathrm{Cov}_X=
 \frac{|\mathcal E_X^{\mathrm{rec}}|}{|\mathcal E_X^{\mathrm{ref}}|}$, rate-weighted coverage $\mathrm{RWC}_X=
 \frac{\sum_{e\in\mathcal E_X^{\mathrm{rec}}}k_e}
      {\sum_{e\in\mathcal E_X^{\mathrm{ref}}}k_e}$, and averaged rate-weighted coverage $\overline{\mathrm{RWC}}=
 \frac1{|\mathcal X|}\sum_{X\in\mathcal X}\mathrm{RWC}_X$ for evaluation. For a fixed reference catalogue, $\mathrm{RWC}_X\to\mathrm{Cov}_X$ as
$T\to\infty$; for $\mathrm{RWC}_X>0$, restricting to the recovered set gives
$1/\mathrm{RWC}_X$ times the full-reference mean waiting time, and
$\overline{\mathrm{RWC}}$ is the expected fraction recovered by the evaluated method when one event is sampled from the reference by rate at each $X$. These are kinetic-focused evaluation
proxies, not full KMC trajectory kinetics (Appendix~\ref{app:rev-rwc}).
All reported RWC values use $T=800\,\mathrm K$.

Throughout the evaluation in this paper, the reference set
$\mathcal E_X^{\mathrm{ref}}$ stays fixed: each output is assigned at most
one matching reference, and repeated hits to the same reference are counted once
(Appendix~\ref{app:rev-matching}). \label{ass:rev-scope}
Throughout the experiments, each $X$ is a SEAKMC active volume of radius
$11\,\text{\AA}$ \citep{xu2012self}, treated as a global state for evaluation. Random Dimer searches are capped at 450 translation steps, whereas ASPIRE
refinement is capped at 300 (Appendix~\ref{app:rev-refinement}). This is because random dimer takes more translations to converge, and the main purpose of ASPIRE is to reduce geometric error and translation steps to achieve computational efficiency, as shown in Figure~\ref{fig:rev-recovery-cost}(b).

\subsection{Reference-event recovery and force cost}
\label{sec:rev-recovery-cost}
\label{sec:rev-ablation}

\begin{figure}[]
    \centering
    \includegraphics[width=\linewidth]
    {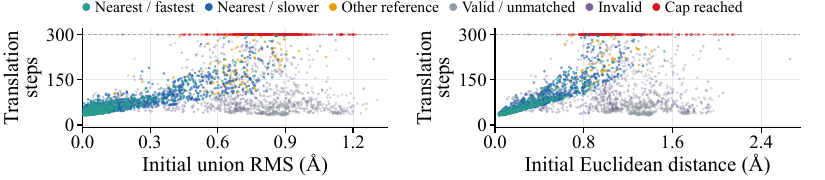}
    \caption{\textbf{Proposal geometry and refinement outcomes.}
    Each point represents a proposal and its subsequent Dimer search.
    Initial error is measured by moving-union RMS
    or Euclidean distance to the nearest reference saddle, and refinement effort by the number of translation steps.
    Colors distinguish nearest-reference recovery (fastest or slower),
    recovery of another reference, valid unmatched outputs,
    invalid outputs, and searches reaching the cap.
    The dashed line marks the 300-step cap.}
    \label{fig:rev-rms-recovery}

\end{figure}

\begin{table}[!b]
    \caption{\textbf{Ablation of Ev-Quiformer.}
    Judge selection is disabled for all variants. 
    O2M denotes the number of sequential matching rounds;
    Uniform denotes all-atom SmoothL1 supervision.  The active-aware objective is specified in Appendix~\ref{app:rev-active}.
    $\bar F_{\mathrm{search}}$ is the mean recorded Dimer search
    force-call count per configuration (Appendix~\ref{app:rev-protocol}).
    Evaluated on 50 test configurations selected with topology-layered sampling covering all five vacancy-connectivity classes (Appendix~\ref{app:rev-eval50}).}
    \label{tab:ablation}
    \centering
    \small
    \setlength{\tabcolsep}{4pt}
    \renewcommand{\arraystretch}{1.08}
    \resizebox{\linewidth}{!}{%
        \begin{tabular}{@{}lccccccrrr@{}}
            \toprule
            & \multicolumn{3}{c}{Architecture}
            & \multicolumn{3}{c}{Supervision}
            & \multicolumn{3}{c}{Results} \\
            \cmidrule(lr){2-4}
            \cmidrule(lr){5-7}
            \cmidrule(lr){8-10}
            Variant
            & Slots
            & Anchor
            & Self-attn.
            & O2M
            & Loss
            & $\beta$ (\AA)
            & Cov.\ (\%) $\uparrow$
            & RWC (\%) $\uparrow$
            & $\bar F_{\mathrm{search}}$ $\downarrow$ \\
            \midrule
            Full Ev-Quiformer
            & 96 & $\checkmark$ & $\checkmark$
            & 4 & Uniform & 0.03
            & \textbf{76.13} & 87.34 & 75,768 \\
            Cross-attention only
            & 96 & $\checkmark$ & $\times$
            & 4 & Uniform & 0.03
            & 75.60 & 84.29 & 71,033 \\
            Single matching round
            & 96 & $\checkmark$ & $\checkmark$
            & 1 & Uniform & 0.03
            & 74.09 & 87.28 & 87,688 \\
            Active-aware loss
            & 96 & $\checkmark$ & $\checkmark$
            & 4 & Active-aware & 0.03
            & 23.29 & 27.18 & 138,302 \\
            Fewer slots
            & 32 & $\checkmark$ & $\checkmark$
            & 4 & Uniform & 0.03
            & 62.10 & 78.78 & \textbf{26,461} \\
            Coarser SmoothL1
            & 96 & $\checkmark$ & $\checkmark$
            & 4 & Uniform & 0.10
            & 74.79 & \textbf{87.75} & 82,460 \\
            No spatial anchor
            & 96 & $\times$ & $\checkmark$
            & 4 & Uniform & 0.03
            & 64.74 & 75.03 & 87,453 \\
            \bottomrule
        \end{tabular}%
    }
\end{table}

\label{sec:rev-assumptions}

Figure~\ref{fig:rev-recovery-curves} compares reference-event
recovery against physical search cost.
ASPIRE achieves higher reference-event coverage and
recovers a comparable fraction of the reference rate mass
at substantially lower search cost compared to 80 random dimer searches. 80 per active volume is chosen as it is a reasonable parameter to get a relatively comprehensive set in SEAKMC deployment.


The judge controls this recovery--cost tradeoff by selecting
proposals before physical refinement. The ASPIRE curves rise rapidly and then flatten, indicating
the possibility of using the judge to further reduce force cost while keeping much of the event set and reference rate mass recovery. Together, these results support learned saddle proposals
and selective refinement as a means of reducing physical
search work while retaining high reference-event recovery.

Figure~\ref{fig:rev-rms-recovery} shows that Geometric error is strongly associated with refinement behavior: proposals with small moving-union RMS typically converge rapidly to their nearest ground-truth saddle, whereas larger errors lead to more variable outcomes. This suggests that geometric error can serve as an effective offline proxy for model quality, consistent with the strong performance of our union-RMS-based judge.



\paragraph{Ablation study.}
Table~\ref{tab:ablation} examines how proposal design affects this
recovery--cost tradeoff. Judge selection is disabled for all variants,
so the comparison does not depend on a confidence threshold.


The full Ev-Quiformer achieves the highest coverage, 76.13\%, with
87.34\% RWC and 75,768 mean search force calls per configuration.
Among those with similar coverage and RWC above 87\%, full Ev-Quiformer reduces the most search force calls.
Together, for the objective of higher coverage and RWC first, and reduced cost second, these results favor the default combination of a larger
proposal set, slot interaction, repeated matching, and uniform
all-atom supervision. Among the tested variants with similar RWC,
the full model provides the highest coverage at the lowest force cost.


\FloatBarrier

\subsection{Wall time and comparison with event reuse}
\label{sec:rev-walltime}

\begin{figure}[!t]
    \centering
    \includegraphics[width=\linewidth]
    {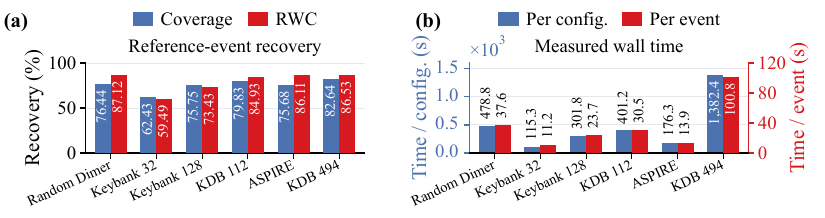}

    \vspace{-10pt}
    \caption{\textbf{Recovery and wall time compared with event reuse.}
    \textbf{(a)} Reference-event coverage and rate-weighted coverage.
    After Random Dimer, methods are ordered by increasing RWC.
    \textbf{(b)} Measured wall time per configuration
    and per recovered reference event.
    Random Dimer and the reuse baselines use 16 CPUs;
    ASPIRE uses 16 CPUs and one A100 GPU. ASPIRE uses judge-based proposal selection with threshold
$\tau=10^{-5}$, refining only proposals with
$\widehat p_m\geq\tau$. Evaluated on 50 test configurations with topology-layered sampling}
    \label{fig:rev-walltime}
 \vspace{-10pt}
\end{figure}

Force-call counts alone do not capture the lookup cost of event reuse. Cost-accounting conventions are discussed in Appendix~\ref{app:rev-matching}.
We therefore compare measured wall time with \emph{adaptations} of KDB and k-ART-style
keybank baselines in Figure~\ref{fig:rev-walltime}. We adapt KDB’s local geometric matching and implement KeyBank using WL hashing, exact graph isomorphism,
  and geometric alignment. KDB transfers aligned saddle coordinates, whereas KeyBank transfers rotated displacement fields; both initialize Dimer refinement. KDB-494 denotes
  494 event templates (Appendix~\ref{app:rev-kdb}); KeyBank-32 and KeyBank-128 use 32 and 128 training configurations (Appendix~\ref{app:rev-keybank}), respectively. Querying the test configurations was already time-consuming at these sizes. We therefore report the measured recovery–wall-time tradeoff at these operating points. 


As shown in Figure~\ref{fig:rev-walltime}, baselines with comparable
RWC are substantially slower than ASPIRE, while faster alternatives
achieve much lower RWC under the stated library and hardware settings. These methods do not contradict each other, so they can be deployed together for complementary benefits.


\section{Conclusion and Limitation}
\label{sec:conclusion}

We introduced ASPIRE, a propose-refine framework for saddle
discovery from a single initial atomic configuration.
Its Ev-Quiformer predictor uses geometry-conditioned event slots to
generate a set of complete displacement fields for refinement on the
original potential. On the four-vacancy bcc Fe benchmark, ASPIRE
achieves higher reference-event coverage than the reported
80-start Random Dimer endpoint, at approximately half the recorded
search-force cost under the respective protocols. The wall-time
comparison further shows lower measured runtime than event-reuse
baselines with comparable RWC under the stated hardware.
These results support learning approximate saddle geometry as
a means of accelerating physical event discovery.

\textbf{Limitation.} Our reference sets are finite and non-exhaustive, and the evaluation
is restricted to four-vacancy bcc Fe under a fixed potential.
Reference-relative coverage and RWC did not establish the accuracy of a full KMC trajectory.
More complete reference catalogues and trajectory-level evaluation
are needed to assess the resulting long-time kinetics.

\clearpage
\begingroup

\hypersetup{linkcolor=red,citecolor=red,urlcolor=red}

\bibliography{iclr2027_conference}
\bibliographystyle{iclr2027_conference}

\appendix

\appendix

\clearpage
\startcontents[appendixoutline]
\startcontents[appendixoutlinecontinued]
\stopcontents[appendixoutlinecontinued]
\begingroup
\color{black}
\hypersetup{linkcolor=black,citecolor=black,urlcolor=black,pdfborder={0 0 0}}
\section*{Supplementary Material}
Appendices~\ref{app:rev-background}--\ref{app:rev-dimer} describe the physical background, dataset construction, neural prediction and screening, and numerical refinement. Appendices~\ref{app:rev-kdb}--\ref{app:rev-trajectory-mds} describe the event-reuse baselines and trajectory visualization. Appendices~\ref{app:rev-theory} and~\ref{app:rev-rwc} provide the theoretical analyses, and Appendix~\ref{app:rev-matching} gives the evaluation protocols and cost accounting. The complete appendix contents follow on two pages.

\subsection*{Appendix Contents}
\begingroup
\small
\setlength{\parskip}{0pt}
\printcontents[appendixoutline]{}{1}{\setcounter{tocdepth}{2}}
\endgroup
\clearpage
\section*{Supplementary Material}
\subsection*{Appendix Contents (continued)}
\begingroup
\small
\setlength{\parskip}{0pt}
\printcontents[appendixoutlinecontinued]{}{1}{\setcounter{tocdepth}{2}}
\endgroup
\endgroup
\clearpage
\section{Physical Background and Coordinate Conventions}
\label{app:rev-background}

This section develops the physical and representational background for
the learning problem. We use $X\in\mathbb R^{N\times3}$ for the initial
configuration, $V$ for the fixed interatomic potential, $S_e$ for a reference
saddle, $D_e=S_e-X$ for its displacement field, and $X_e^+$ for the connected
product minimum. Parent-crystal, periodic-cell, and search-coordinate
notation specifies how these quantities are constructed.
\subsection{bcc Fe, vacancies, and the fixed interatomic potential}
\label{app:rev-fe-background}

\paragraph{Parent lattice.}
In conventional cubic axes, an ideal body-centered-cubic (bcc) lattice with
lattice parameter $a_0$ can be written as
\begin{equation}
 \Lambda_{\mathrm{bcc}}
 =a_0\mathbb Z^3\;\cup\;
 a_0\!\left(\mathbb Z^3+(\tfrac12,\tfrac12,\tfrac12)\right).
 \label{eq:app-bcc-lattice}
\end{equation}
There are two lattice sites per conventional cubic cell. The elemental bcc
structure has space group $Im\overline3m$ (No.~229)
\citep{mehl2017aflow}. An equivalent integer-coordinate description is
\begin{equation}
 \Lambda_{\mathrm{bcc}}
 =\frac{a_0}{2}
 \left\{z\in\mathbb Z^3:
 z_1\equiv z_2\equiv z_3\pmod 2\right\}.
 \label{eq:app-bcc-integer}
\end{equation}
Thus the integer coordinates are either all even or all odd, enabling exact
comparison of parent vacancy-site patterns. Relaxed atomic positions and
saddle candidates are represented by continuous Cartesian coordinates.

\paragraph{Vacancies and competing rearrangements.}
A vacancy is an unoccupied site of the parent lattice.
Equation~\revEqref{eq:app-bcc-lattice}
gives eight nearest-neighbor vectors from a lattice site,
\begin{equation}
 \frac{a_0}{2}(\pm1,\pm1,\pm1),
 \qquad d_{\mathrm{nn}}=\frac{\sqrt3}{2}a_0.
 \label{eq:app-bcc-neighbors}
\end{equation}
These directions describe nearest-neighbor atom--vacancy exchange.
Multiple vacancies alter local coordination and induce relaxation of the
surrounding Fe atoms, producing rearrangements with different saddle
geometries and barriers. SEAKMC searches directly in atomic coordinate
space to discover the corresponding mechanisms \citep{xu2012self}.

The prediction target is a complete $N\times3$ displacement field, which
describes both motion toward vacant sites and surrounding atomic relaxation.
This representation accommodates varying numbers of participating atoms
and multiple accessible rearrangements for the same vacancy configuration.

\paragraph{M07 and the EAM energy.}
The main-text experiments use the M07 embedded-atom-method (EAM) potential
for Fe. The NIST repository identifies the standard M07 file with the Fe
potential developed by Marinica in 2007 using a point-defect-oriented fitting
database, and associates it with the studies of
\citet{malerba2010comparison} and \citet{marinica2012irradiation}
. For a single-element EAM model, the energy of a physical
configuration $Y$ has the form
\begin{equation}
 V_{\mathrm{full}}(Y)
 =\sum_i F_{\mathrm{emb}}\!\left(\sum_{j\ne i}
       f_{\mathrm{dens}}(r_{ij})\right)
  +\frac12\sum_i\sum_{j\ne i}\phi_{\mathrm{pair}}(r_{ij}),
 \label{eq:app-eam}
\end{equation}
where $r_{ij}$ is evaluated with the physical boundary convention,
$F_{\mathrm{emb}}$ is the embedding function, $f_{\mathrm{dens}}$ contributes
to the local density, and $\phi_{\mathrm{pair}}$ is the pair interaction
\citep{daw1984embedded}. Here $Y$ distinguishes the physical system from
the local array $X$ supplied to the predictor.

The neural model proposes candidate geometries, while the fixed Fe
potential supplies energies and forces for initial relaxation, Dimer
refinement, and barrier evaluation. This division connects learned
proposals to a consistent physical energy surface throughout the benchmark.

\subsection{Periodic images and a local Cartesian representation}
\label{app:rev-coordinate-background}
\label{app:rev-periodic-background}

\paragraph{Periodic images are translated copies.}
A periodic simulation cell represents a crystal together with translated
copies of that cell \citep{abraham_2026_20845233}. Let $H_{\mathrm{sys}}$ be the
nonsingular matrix whose rows are the physical cell vectors, and use row
vectors for positions. Then $r$ and $r+nH_{\mathrm{sys}}$ represent the same
indexed atom for an allowed integer lattice shift $n$. In a partially
periodic system, components of $n$ along nonperiodic axes are zero.

For a displacement row vector $v$, the minimum-image operation is
\begin{equation}
 \operatorname{MIC}_{H_{\mathrm{sys}}}(v)
 =v-n^*H_{\mathrm{sys}},\qquad
 n^*\in\operatorname*{arg\,min}_{n\in\mathbb Z^3_{\mathrm{per}}}
       \|v-nH_{\mathrm{sys}}\|_2.
 \label{eq:app-mic-background}
\end{equation}
It chooses a shortest representative among the allowed translations.
For an orthogonal, fully periodic cell, this becomes
\begin{equation}
 \operatorname{MIC}_{H_{\mathrm{sys}}}(v)
 =\bigl(vH_{\mathrm{sys}}^{-1}
       -\operatorname{round}(vH_{\mathrm{sys}}^{-1})\bigr)
       H_{\mathrm{sys}},
 \label{eq:app-mic-orthogonal}
\end{equation}
apart from image ties. The row-vector convention gives the inverse
$H_{\mathrm{sys}}^{-1}$. For a general skew cell, Equation~\revEqref{eq:app-mic-background}
defines the shortest-image search. At exact ties, preprocessing uses a fixed
image-selection convention compatible with the stated symmetry conditions.

\paragraph{From stored positions to $X$.}
A local representation places atoms near a defect in a common image even
when their stored coordinates straddle the periodic box. Let $c$ be a
local reference point, also written as a row vector, and let $i(a)$ identify
local atom $a$ in the physical configuration $Y$. A suitable local
representation is
\begin{equation}
 X_a=\operatorname{MIC}_{H_{\mathrm{sys}}}\bigl(Y_{i(a)}-c\bigr).
 \label{eq:app-local-representation}
\end{equation}
Alternatively, an already unwrapped local configuration can be supplied
directly. In either case, $X$ is the resulting Cartesian array with fixed
atom order; cell information and $c$ specify its construction.

Pairwise Cartesian differences agree with minimum-image displacements
under the following sufficient geometric condition. Let
\[
 \ell_{\min}
 =\min_{n\in\mathbb Z^3_{\mathrm{per}}\setminus\{0\}}
      \|nH_{\mathrm{sys}}\|_2,
\]
and suppose all represented positions satisfy $\|X_a\|_2\le r_{\mathrm{AV}}$.
If
\begin{equation}
 4r_{\mathrm{AV}}<\ell_{\min},
 \label{eq:app-local-image-condition}
\end{equation}
then all pair differences in the local array are their unique shortest
periodic representatives. Indeed, $\|X_a-X_b\|_2\le2r_{\mathrm{AV}}$, whereas
any nonzero lattice translation $l$ gives
$\|X_a-X_b-l\|_2\ge\ell_{\min}-2r_{\mathrm{AV}}>2r_{\mathrm{AV}}$.
This criterion makes the relation between the local Cartesian chart and
minimum-image geometry explicit in terms of the actual cell dimensions.

\paragraph{Operations used by the main-text model.}
After the local representation is fixed, the encoder uses $X_a-X_b$ in its
$6\,\text{\AA}$ radius graph. The decoder uses
\begin{equation}
 b_m=\sum_a A_{ma}X_a,\qquad X_a-b_m,
 \qquad \sum_a A_{ma}=1.
 \label{eq:app-cartesian-anchor}
\end{equation}
The anchor is the attention-weighted mean in the supplied local Cartesian
representation. Preparing a consistent local image before this averaging
preserves the intended neighborhood geometry; for example, raw positions
$0.49L$ and $-0.49L$ straddle a periodic boundary and have nearest-image
separation $0.02L$.

\paragraph{Displacement targets and atom correspondence.}
The initial and saddle arrays must refer to the same atoms in compatible
representations before forming
\begin{equation}
 D_e=S_e-X,\qquad \widehat S_m=X+\widehat D_m.
 \label{eq:app-displacement-convention}
\end{equation}
Image tracking along a search gives a consistent physical displacement.
For a localized event, compatible shortest displacements can be selected
relative to the initial atoms; trajectories winding around the cell use
path unwrapping.

With consistent atom correspondence and image representation, assignment
and regression use the full Cartesian residual $\widehat D_m-D_e$ directly.
The same atom ordering is retained in $X$, $S_e$, and their displacement field.

\paragraph{Three different meanings of ``image.''}
Periodic images are translated representatives of the same atom.
The reflections and inversion used in cubic-symmetry deduplication are
spatial transformations of an entire vacancy pattern. The two images of a
Dimer are nearby trial configurations in configuration space. Each operation
therefore has a defined role in data preparation or numerical search.

\subsection{Active volumes, surrounding atoms, and free coordinates}
\label{app:rev-active-volume-background}
\label{app:rev-free-coordinates}

SEAKMC localizes event searches to active volumes surrounding defects,
reducing the coordinates involved in a search while retaining an atomistic
description of the rearrangement \citep{xu2012self,XU2015135,LIANG2023112354}. In this paper,
$X$ represents an active volume of radius $r_{\mathrm{AV}}=11\,\text{\AA}$.
Its $N$ local atoms are selected from the 128,000-atom parent block before
vacancy removal. The active-volume radius controls the search region,
the $6\,\text{\AA}$ graph cutoff controls neural neighborhoods, and the
potential's interaction cutoff controls physical interactions.

The active-volume model incorporates its prescribed surrounding environment.
With fixed surroundings, a local trial array $Z$ can be embedded into the
parent configuration by a map $\iota_X$. The local potential is
\begin{equation}
 V(Z)=V_{\mathrm{full}}(\iota_X Z)-C_X,
 \label{eq:app-restricted-potential}
\end{equation}
where $C_X$ is any configuration-independent energy offset. Interactions
across the active-volume boundary are included according to the physical
model. The offset cancels from $V(S_e)-V(X)$. This physical locality model
specifies the forces and barriers used for refinement \citep{beland2015kinetic}.

When some represented atoms are fixed during refinement, let
$\mathcal F_X$ index the free atoms. An admissible search displacement then
has $D_a=0$ for $a\notin\mathcal F_X$, and the search acts on vectorized
free coordinates. Any additional linear constraints are imposed before
taking the gradient and Hessian of the restricted potential. This is the
coordinate interpretation used by the local-convergence proposition.

The free-coordinate representation specifies the nondegeneracy conditions.
A translation-invariant whole-system potential has uniform translation zero
modes; fixing the environment or removing these modes supplies the space in
which a positive-definite minimum Hessian or nondegenerate index-one saddle
is defined. Motion of the free atoms relative to a fixed shell remains part
of the local physical problem.

\subsection{Dimer saddle searches and learned initialization}
\label{app:rev-dimer-background}

\paragraph{What Dimer searches for.}
The Dimer method seeks an index-one saddle using first derivatives of the
potential, without requiring a product minimum as input
\citep{henkelman1999dimer}. Work in the same vectorized free coordinates as
above, and denote a trial point by $y$ and its force by $f(y)=-\nabla V(y)$.
Dimer combines ascent along an unstable direction with relaxation in
transverse directions to locate an index-one saddle.

\paragraph{Two configurations estimate directional curvature.}
A symmetric description of a Dimer consists of a center $y$, a unit
orientation $n$, and a half-separation $h_{\mathrm{dim}}>0$:
\begin{equation}
 y^+=y+h_{\mathrm{dim}}n,\qquad
 y^-=y-h_{\mathrm{dim}}n,\qquad \|n\|_2=1.
 \label{eq:app-dimer-images}
\end{equation}
The orientation assigns a displacement direction to every free coordinate.
A finite difference of the forces estimates a Hessian--vector product,
\begin{equation}
 B_{h_{\mathrm{dim}}}(y,n)
 =-\frac{f(y^+)-f(y^-)}{2h_{\mathrm{dim}}}
 \longrightarrow \nabla^2V(y)n
 \quad\text{as }h_{\mathrm{dim}}\to0,
 \label{eq:app-dimer-hvp}
\end{equation}
with estimated curvature
$\kappa_{h_{\mathrm{dim}}}=n^\top B_{h_{\mathrm{dim}}}(y,n)$.
This explains how curvature information can be obtained without assembling
a full Hessian. The symmetric formula describes the method's principle;
implementations can use force reuse or one-sided estimates to reduce the
number of force evaluations \citep{kastner2008superlinearly}.

\paragraph{Rotation and translation are different updates.}
At fixed $y$, the orientation is adjusted toward a low-curvature mode.
In the small-separation limit, minimizing
$n^\top\nabla^2V(y)n$ over the unit sphere gives a tangential gradient
proportional to
\begin{equation}
 (I-nn^\top)\nabla^2V(y)n,
 \label{eq:app-dimer-rotation}
\end{equation}
which can be estimated using Equation~\revEqref{eq:app-dimer-hvp}.
A translation update then moves the center. In a negative-curvature
region, the minimum-mode-following effective force is
\begin{equation}
 f_{\mathrm{eff}}(y,n)
 =f(y)-2n\bigl(n^\top f(y)\bigr)
 =-(I-2nn^\top)\nabla V(y).
 \label{eq:app-dimer-effective-force}
\end{equation}
The force component parallel to $n$ is reversed while perpendicular
components are retained. If $n$ is the unstable mode near an index-one
saddle, this reverses the downhill motion along that mode and leaves
relaxation in stable directions. Rotation and translation are repeated,
with numerical step control and convergence tests. Additional controls
handle positive-curvature escape from the initial basin
\citep{chill2014eon}.

\paragraph{Random and learned initialization.}
The conventional path initializes the search direction randomly, using the
axis sampler in Appendix~\ref{app:rev-refinement}. ASPIRE supplies
$\widehat S_m=X+\widehat D_m$ as the proposed center configuration.
In free coordinates, this gives
\begin{equation}
 y_0=\operatorname{vec}\!\left(\widehat S_m|_{\mathcal F_X}\right).
 \label{eq:app-dimer-learned-center}
\end{equation}
The Dimer also requires an orientation. For displacement-derived
initialization, its general form is
\begin{equation}
 n_0=\frac{\mathsf T_X(\widehat D_m)}
              {\|\mathsf T_X(\widehat D_m)\|_2},
 \qquad \|\mathsf T_X(\widehat D_m)\|_2>0,
 \label{eq:app-dimer-initial-orientation}
\end{equation}
where $\mathsf T_X$ collects the orientation preprocessing and restriction
to search coordinates. In the supplied Cartesian path, the bridge removes the mean displacement
for the initial axis, retains the full displacement for the center preload,
and can reset the axis during loading
(Appendix~\ref{app:rev-refinement}). The center preload specifies a candidate
geometry, while orientation updates estimate the local unstable direction.

Proposition~\ref{prop:rev-ideal-radius} analyzes the exact minimum-mode
iteration with its specified step size. The resulting local convergence
region provides a mathematical rationale for geometric initialization.
Appendix~\ref{app:rev-refinement} gives the finite-difference orientation
updates and search controls used in numerical refinement.

\paragraph{Termination, saddle validation, and connectivity.}
The target in Section~\ref{sec:revised-problem} combines stationarity,
the index-one condition, and downhill connectivity from
$S_e\pm\epsilon v_1(S_e)$ to $X$ and a distinct minimum $X_e^+$.
Together these conditions define the transition tuple and barrier
\begin{equation}
 e=(X,S_e,X_e^+,\Delta E_e),\qquad
 \Delta E_e=V(S_e)-V(X)>0.
 \label{eq:app-physical-event}
\end{equation}
The main-text idealized map $\mathcal R_X$ encodes these target conditions.
The native Dimer/SEAKMC implementation uses force/step stopping, two-sided
relaxation, final flags, and cap checks as its operational acceptance
protocol (Appendix~\ref{app:rev-refinement}).

\paragraph{Search work in this benchmark.}
The main-text evaluation caps ASPIRE refinement at 300 translation steps
and random-start Dimer at 450. The predictor has $M=96$ slots; with judge
selection, fewer than $M$ proposals may be submitted. The random baseline
uses up to 80 searches per configuration. The slot count specifies the
proposal pool, the translation cap controls each search, and force-call
counts measure its computational work, including orientation updates.

The Dimer search-force ledger counts calls routed through its force
routine for recorded searches, including failed and duplicate searches.
It excludes model inference and force evaluations attributed to separate
setup and downhill relaxation. Wall time is reported separately. Dimer
separation, rotation limits, force tolerances, and step-size controls are
run-specific settings described in the numerical refinement protocol.

\subsection{Transition rates and kinetic Monte Carlo}
\label{app:rev-kmc-background}
\label{app:rev-kmc}

\paragraph{From a saddle to an event rate.}
Transition-state theory associates a rate with escape through a transition
state. Under its harmonic approximation, an event rate has Arrhenius form
\begin{equation}
 k_e=\nu_e\exp\!\left(-\frac{\Delta E_e}{k_{\mathrm B}T}\right).
 \label{eq:app-htst-rate}
\end{equation}
The prefactor $\nu_e$ accounts for the local vibrational contribution. For
$d$ unconstrained vibrational degrees of freedom after zero modes have
been removed, the classical harmonic expression is
\begin{equation}
 \nu_e=\frac1{2\pi}
 \frac{\prod_{i=1}^{d}\omega_i(X)}
      {\prod_{i=2}^{d}\omega_i(S_e)},
 \label{eq:app-vineyard-prefactor}
\end{equation}
where $\omega_i$ are angular frequencies from the mass-weighted Hessian
and the unstable saddle mode is omitted from the denominator
\citep{vineyard1957frequency}. The benchmark uses the common-prefactor
convention $k_e=\nu\exp[-\Delta E_e/(k_{\mathrm B}T)]$ for reference weights.
The common factor cancels from RWC and relative event probabilities;
absolute waiting times retain its scale.

\paragraph{State selection and the clock.}
The jump-process description assumes that relaxation within a metastable
basin is fast relative to escape, so that event rates can be treated as
constant while the system remains in that state. KMC advances the system
through discrete transitions using their event rates \citep{voter2007introduction}.
For a nonempty event set $\mathcal E_X$ with total rate
\begin{equation}
 \Lambda_X=\sum_{e\in\mathcal E_X}k_e,
 \label{eq:app-kmc-total-rate}
\end{equation}
a residence-time step draws independent
$u_1,u_2\sim\operatorname{Uniform}(0,1)$ and uses
\begin{equation}
 \begin{gathered}
 e^*=\min\!\left\{j:\sum_{e=1}^{j}k_e\ge u_1\Lambda_X\right\},
 \qquad X\leftarrow X_{e^*}^+,\\
 \Delta t=-\frac{\log u_2}{\Lambda_X}.
 \end{gathered}
 \label{eq:app-kmc-step}
\end{equation}
Event indices can be ordered arbitrarily for the cumulative sum.
The event probability is $k_e/\Lambda_X$, and the physical waiting time
is exponential with mean $1/\Lambda_X$. Event selection and clock advancement
use independent random draws \citep{voter2007introduction}.

\paragraph{On-the-fly discovery and SEAKMC.}
When the accessible products are unknown, an on-the-fly method first
constructs an event catalogue at the current state. SEAKMC combines defect
identification, local active-volume searches, transition processing, and
KMC event selection \citep{xu2012self,XU2015135}. After a selected transition,
the affected local environments and event catalogues are updated.
Repeated event discovery therefore motivates efficient proposal generation
across the configurations encountered during simulation.

ASPIRE targets the proposal-and-refinement stage of this process.
Its experiments evaluate distinct reference-event recovery and search
efficiency from fixed initial configurations. These measurements assess
the saddle-discovery component that supplies mechanisms for downstream
event-based modeling.

\paragraph{Why distinct events and rate mass both matter.}
Each distinct physical exit channel contributes to the total escape rate.
Repeated searches reaching one event count as one recovered identity, while
symmetry-related exits with equal rates can contribute separately.
Combining channels into a common catalogue entry therefore retains their
summed rates or multiplicities (Appendix~\ref{app:rev-dataset-event-distinction}).

For a fixed reference set and its recovered subset,
\begin{equation}
 \Lambda_X^{\mathrm{ref}}=\sum_{e\in\mathcal E_X^{\mathrm{ref}}}k_e,
 \qquad
 \Lambda_X^{\mathrm{rec}}=\sum_{e\in\mathcal E_X^{\mathrm{rec}}}k_e,
 \qquad
 \mathrm{RWC}_X=\frac{\Lambda_X^{\mathrm{rec}}}{\Lambda_X^{\mathrm{ref}}}.
 \label{eq:app-reference-rate-mass}
\end{equation}
Keeping the original reference rates, a one-step model restricted to the
recovered subset has mean waiting time larger by $1/\mathrm{RWC}_X$ when
$\mathrm{RWC}_X>0$. Also, a rate-weighted draw from the reference set belongs
to the recovered subset with probability $\mathrm{RWC}_X$. These are the
single-state interpretations used in the main text. An empty recovered
set has no finite escape time under that restricted model.

The benchmark measures distinct reference-event recovery and search work
across held-out states. Coverage summarizes recovered event identities,
while RWC quantifies their reference escape-rate mass. These metrics provide
complementary measures of the quality and efficiency of saddle discovery.

\subsection{Interface to the main-text learning problem}
\label{app:rev-background-interface}

The physical environment, fixed potential, and consistent Cartesian
representation define the input $X$ and targets $D_e=S_e-X$. The predictor
maps $X$ to $\{\widehat D_m\}_{m=1}^{M}$ without product endpoints.
The encoder's $X_a-X_b$, decoder's $X_a-b_m$, full SmoothL1 residual,
and reference matcher use this shared representation, connecting the
physical task to geometric learning and refinement.

\section{Dataset Construction and bcc-Symmetry Deduplication}
\label{app:rev-datasets}

\subsection{Physical construction and recorded examples}
\label{app:rev-dataset-records}

Both collections start from a periodic $40\times40\times40$ conventional-cell
bcc Fe block containing 128,000 atoms before vacancy removal. One to four
ideal lattice sites are removed, and the parent is relaxed through SEAKMC
under the M07 EAM potential. The lattice parameter is approximately
$a_0=2.85528\,\text{\AA}$ and the cubic box length is approximately
$114.2112\,\text{\AA}$; the stored cell specifies the exact geometry.
Forces and energies are evaluated with LAMMPS using \texttt{eam/fs} and
\texttt{M07\_eam.fs}. Positions and energies use \AA{} and eV,
respectively; the physical background is given in
Appendix~\ref{app:rev-fe-background}.

The learning input is the relaxed local active-volume (AV) array $X$ of
radius $11\,\text{\AA}$. Each example pairs $X$ with a variable-size,
unordered reference event set collected through 80 random-start Dimer
searches and SEAKMC validation, retaining barriers in
$[0.1,2.0]\,\mathrm{eV}$. Table~\ref{tab:app-dataset-summary} gives the
dataset sizes and their roles in this paper. All reported evaluations use
\textsc{BccFe-4vacAV-4000}; \textsc{BccFe-1to4Vac} is constructed for future
KMC trajectory simulation and analysis.

\begin{table}[t]
 \centering
 \small
 \setlength{\tabcolsep}{4pt}
 \caption{\textbf{Dataset records.} Each state is one retained local AV.
 The two collections use their respective AV-selection protocols and
 have complementary roles in evaluation and future trajectory studies.}
 \label{tab:app-dataset-summary}
 \begin{tabular}{@{}lrrcl@{}}
 \toprule
 Dataset & AV states & Events & Vacancies & Use \\
 \midrule
 \textsc{BccFe-4vacAV-4000} & 4,000 & 65,450 & 4 & Reported evaluation \\
 \textsc{BccFe-1to4Vac} & 5,372 & 83,372 & 1--4 & Future trajectory studies \\
 \bottomrule
 \end{tabular}
\end{table}

\paragraph{Reference-collection settings.}
The common generation settings are summarized in
Table~\ref{tab:app-dataset-generation-settings}. The active, buffer, and
fixed settings describe the physical AV construction; they are distinct
from the encoder's graph cutoff. With \texttt{SearchBuffer=False},
the search coordinates comprise the active region.

\begin{table}[t]
 \centering\small
 \caption{\textbf{Reference-catalogue generation settings.}
 Dataset collection uses the following protocol; evaluation settings
 are specified in Appendix~\ref{app:rev-refinement}.}
 \label{tab:app-dataset-generation-settings}
 \begin{tabular}{@{}p{0.57\linewidth}p{0.35\linewidth}@{}}
 \toprule
 Setting & Value \\
 \midrule
 Point-defect-reduction cutoff, \texttt{DCut4PDR} & $4.0\,\text{\AA}$ \\
 \texttt{DActive}, \texttt{DBuffer}, \texttt{DFixed} & $11,1,11\,\text{\AA}$ \\
 Dimer starts, \texttt{NSearch} & 80 per searched AV \\
 Buffer search, \texttt{SearchBuffer} & False \\
 Maximum translation iterations, \texttt{NMax4Trans} & 900 \\
 Trial / maximum step size & $0.015/0.05\,\text{\AA}$ \\
 Dimer separation & $0.005\,\text{\AA}$ \\
 Local relaxation & Enabled \\
 Barrier window & $[0.1,2.0]\,\mathrm{eV}$ \\
 Connectivity checking & Enabled \\
 Prefactor calculation, \texttt{CalPrefactor} & False \\
 \bottomrule
 \end{tabular}
\end{table}

Reference collection uses a 900-translation limit. The main-text
evaluation uses caps of 300 for ASPIRE refinement and 450 for conventional
Dimer (Appendix~\ref{app:rev-refinement}); search records belong to their
respective collection or evaluation stage.

\subsection{Four-vacancy benchmark construction}
\label{app:rev-4vac-construction}

\paragraph{Vacancy topology enumeration and parent sampling.}
For \textsc{BccFe-4vacAV-4000}, vacancies occupy ideal bcc sites represented
as $(a_0/2)(i,j,k)$, with the three integers either all even or all odd.
Candidate four-site patterns have maximum pairwise vacancy distance
strictly below $11\,\text{\AA}$. Identifying patterns under periodic lattice
translations and all 48 operations of the cubic point group $O_h$ gives
21,297 symmetry-distinct candidate topologies. The symmetry operations and
a mathematical canonical key are described below.

Parent sampling selects 4,000 configurations using seed \texttt{20260818}.
All candidates in the point-defect-reduction (PDR) fragmentation classes
$4$, $3+1$, and $2+2$ are retained; the remaining classes are sampled across
24 diameter bins. The selected class counts are given in
Table~\ref{tab:app-4vac-sampling}. The seed and stratification specify
parent selection for the 4,000-configuration corpus.

\begin{table}[t]
 \centering\small
 \caption{\textbf{Four-vacancy parent sampling.} Counts are for the
 recorded PDR fragmentation signatures used in dataset construction.}
 \label{tab:app-4vac-sampling}
 \begin{tabular}{@{}lr@{}}
 \toprule
 PDR fragmentation signature & Selected parents \\
 \midrule
 $4$ & 35 \\
 $3+1$ & 583 \\
 $2+2$ & 246 \\
 $2+1+1$ & 1,254 \\
 $1+1+1+1$ & 1,882 \\
 \midrule
 Total & 4,000 \\
 \bottomrule
 \end{tabular}
\end{table}

PDR denotes point-defect reduction in the defect-detection procedure.
The table records the parent-sampling categories; selection of the fixed
50-state evaluation subset is described separately in
Appendix~\ref{app:rev-eval50}.

\paragraph{One-AV Wigner--Seitz construction.}
For each relaxed parent, Wigner--Seitz (WS) detection against the full
perfect reference first determines the possible PDR groups. A target group
is selected that reproduces itself when the other vacancy markers are
hidden. The physical parent retains all four vacancies. The case-specific
WS reference is formed by removing the non-target vacancy sites from the
perfect reference, so only the target-group vacancies appear as WS
differences. This yields exactly one selected AV per parent, with all four
physical vacancies inside its active sphere. The WS matching radius is
$1.03785\,\text{\AA}$, obtained from
$\texttt{DCut4Def}\times\texttt{cutdefectmax}=0.3\times3.4595\,\text{\AA}$,
and \texttt{NPredef=0}.

Selecting one AV preserves its reference labels during collection.
The reduced WS reference defines defect detection while retaining the
physical parent and potential-energy surface. The resulting local-environment
corpus contains 65,450 connected events, with 2--43 per AV and a mean of 16.36.

\subsection{One-to-four-vacancy collection}
\label{app:rev-1to4-construction}

\paragraph{Candidate coverage.}
The initial candidate bank contains 8,192 ideal vacancy arrangements:
1, 95, 2,048, and 6,048 with one, two, three, and four vacancies,
respectively. Candidates are filtered under translations and all 48
$O_h$ operations. Coverage includes compact clusters, fragmented motifs,
spatially separated arrangements, and directly constructed vacancy-hop
endpoints. Hop-related descendants retain their source family assignment.
This designed candidate bank provides broad structural coverage for the
subsequent AV construction and release filtering.

\paragraph{Native CN construction and independent WS audit.}
Production AVs use native coordination-number (CN) defect detection,
\texttt{NPredef=1}, native PDR, and the shared $11/1/11\,\text{\AA}$ settings.
WS is run only on a deep copy of the relaxed full parent against the full
perfect reference to audit vacancy content. An eligible native AV must
contain all intended parent vacancies in its active region, have the
expected vacancy count, and have no WS interstitial. Eligibility is
evaluated at the native center against the full perfect reference;
parents without an eligible native AV are recorded as ineligible.

The primary AV maximizes the minimum vacancy-to-active-boundary margin,
with randomized ties within $0.25\,\text{\AA}$ bands. Approximately one ninth
of topology families may contribute a second distinct native AV, with at
most two searched AVs per parent. This policy builds a structurally diverse
collection with explicit parent and family provenance.

\paragraph{Full-geometry deduplication and event collection.}
Before Dimer search, complete active, buffer, and fixed coordinates are
compared under translations and all 48 cubic operations using
species-preserving bijective matching at $0.001\,\text{\AA}$ tolerance.
A shared registry prevents concurrent searches of congruent geometries.
Subset selection preserves the native coordinates, atom regions, AV
indices, and center provenance. Native per-AV search and postprocessing
use 80 starts, and labels are archived before cross-AV deletion. The
collection wrapper stops before KMC event selection and propagation.

The released collection contains 5,372 deduplicated AVs from 4,812 parent
IDs and 83,372 events, distributed as in
Table~\ref{tab:app-1to4-distribution}. The number of AV examples is therefore
different from the number of distinct parent IDs.

\begin{table}[t]
 \centering\small
 \caption{\textbf{One-to-four-vacancy release.}
 Counts summarize the final AV collection after construction and filtering.}
 \label{tab:app-1to4-distribution}
 \begin{tabular}{@{}rrr@{}}
 \toprule
 Vacancies & Released AVs & Events \\
 \midrule
 1 & 1 & 12 \\
 2 & 38 & 492 \\
 3 & 1,121 & 16,073 \\
 4 & 4,212 & 66,795 \\
 \midrule
 Total & 5,372 & 83,372 \\
 \bottomrule
 \end{tabular}
\end{table}

\subsection{Reference-event records and corpus distinction}
\label{app:rev-dataset-reference-records}

Each example groups the retained reference events by their shared initial
AV, providing direct supervision for the one-to-many prediction task. Core
labels include initial, saddle, and final positions; barriers; energy
biases; connectivity; and event/search identifiers. Vacancy topology, AV
selection, and parent/family metadata record provenance. Compatible atom
ordering and coordinate representations are retained when forming
$D_e=S_e-X$, as described in
Appendix~\ref{app:rev-coordinate-background}. Vacancy patterns record
construction provenance. Reference rates use the common-prefactor convention,
consistent with \texttt{CalPrefactor=False} during collection.

The full \textsc{BccFe-1to4Vac} representation retains complete AV geometry
and region masks. Its compact representation stores active-atom arrays
and keeps the static buffer/fixed coordinates once in
\texttt{context\_pos0}; original ordering and metadata are retained under
\texttt{source\_\_*}. The provided \texttt{compact\_io.load\_full\_av}
reader reconstructs the full representation, preserving the input $X$
and displacement fields used by the learning objective.

The 80-start collection protocol and barrier window define the observed
reference catalogue for each state. Deduplication retains distinct events,
and evaluation reports recovery relative to these catalogues.

The corpora provide complementary local-environment collections. The
four-vacancy benchmark has one WS-selected AV per parent; the one-to-four-vacancy
collection uses native-CN AVs with an independent WS audit and can retain
a second AV. Each corpus is deduplicated using its construction protocol.

\subsection{The 48 cubic operations used for configuration deduplication}
\label{app:rev-bcc48}

Configuration deduplication uses the 48 operations of the full cubic point
group $O_h=m\overline3m$, comprising 24 proper rotations and 24 improper operations. In conventional
cubic coordinates, it has the explicit matrix representation
\begin{equation}
 \mathcal G_{\mathrm{bcc}}
 =\left\{G=\operatorname{diag}(\varsigma_1,\varsigma_2,\varsigma_3)C_\pi:
 \varsigma_i\in\{-1,+1\},\ \pi\in\mathfrak S_3\right\},
 \qquad |\mathcal G_{\mathrm{bcc}}|=2^3\,3!=48,
 \label{eq:app-bcc48}
\end{equation}
where $C_\pi$ is a $3\times3$ coordinate-permutation matrix. This is
different from the $N\times N$ atom-permutation matrix $P$ used in the
main-text equivariance statement.

The six coordinate orders are
\begin{equation}
 \begin{gathered}
 (x,y,z),\ (x,z,y),\ (y,x,z),\\
 (y,z,x),\ (z,x,y),\ (z,y,x).
 \end{gathered}
 \label{eq:app-bcc-coordinate-orders}
\end{equation}
For each order, all eight independent sign choices are included. Every
matrix in Equation~\revEqref{eq:app-bcc48} is orthogonal. Its determinant is
$(\varsigma_1\varsigma_2\varsigma_3)\operatorname{sgn}(\pi)$, so 24 matrices
have determinant $+1$ and 24 have determinant $-1$. The latter include
reflections and inversion combined with rotations; for example,
$\operatorname{diag}(-1,1,1)$ reflects one coordinate and $-I$ is inversion.

These matrices preserve the bcc lattice: a coordinate permutation and sign
changes preserve the all-even or all-odd condition in
Equation~\revEqref{eq:app-bcc-integer}. The point group is finite, while
lattice translations belong to the space-group description and are handled
separately. A symmetric vacancy pattern can have fewer than 48 distinct
transformed representatives because multiple group operations may yield
the same unordered set.

\subsection{Equivalence of vacancy patterns and an explicit canonical key}
\label{app:rev-bcc-canonical}

Let $n_{\mathrm v}$ be the number of initial vacancies and let
$\mathcal V=\{v_1,\ldots,v_{n_{\mathrm v}}\}\subset\Lambda_{\mathrm{bcc}}$
be their unordered parent-lattice positions, represented as row vectors.
The vacancy-pattern reduction uses one common cubic operation and a
lattice translation. In a consistent local representation, the equivalence
criterion is
\begin{equation}
 \mathcal V'\sim\mathcal V
 \quad\Longleftrightarrow\quad
 \exists G\in\mathcal G_{\mathrm{bcc}},\ t\in\Lambda_{\mathrm{bcc}}:
 \mathcal V'=\{vG^\top+t:v\in\mathcal V\}.
 \label{eq:app-vacancy-equivalence}
\end{equation}
A common transform acts on the whole pattern and preserves its geometry.
Translation removal defines the comparison key; the stored AV retains its
native coordinate representation.

For completeness, one exact canonicalization of
Equation~\revEqref{eq:app-vacancy-equivalence} is as follows. Write
$z_i=2v_i/a_0\in\mathbb Z^3$. For each of the 48 matrices and each choice of
vacancy $j$ as a translation anchor, form the unordered relative set
\begin{equation}
 \mathcal Z_{G,j}
 =\{(z_i-z_j)G^\top:i=1,\ldots,n_{\mathrm v}\}.
 \label{eq:app-canonical-relative-set}
\end{equation}
Sort its rows lexicographically, concatenate the sorted triples, and take
the lexicographically smallest resulting integer tuple:
\begin{equation}
 \operatorname{key}_{\mathrm{bcc}}(\mathcal V)
 =\min_{\substack{G\in\mathcal G_{\mathrm{bcc}}\\1\le j\le n_{\mathrm v}}}
   \operatorname{lex}\!\left(
    \operatorname{sort}_{\mathrm{row}}(\mathcal Z_{G,j})\right).
 \label{eq:app-bcc-key}
\end{equation}
Sorting handles the absence of vacancy labels; enumeration of all anchors
removes dependence on which vacancy was chosen as the origin. For four
vacancies, this construction inspects at most $48\times4=192$ candidate
relative sets, accounting for all four choices of translation anchor.

To see why the key works, a lattice translation cancels in every
$z_i-z_j$, a vacancy reordering only permutes candidate anchors and sorted
rows, and applying a cubic operation permutes the 48 group elements by
composition. Hence equivalent patterns have equal keys. Conversely,
equality of two such sorted relative sets supplies a common cubic
transformation and translation between their parent vacancy positions.
This construction therefore realizes the stated equivalence without a
floating-point tolerance on ideal lattice coordinates.

Equation~\revEqref{eq:app-bcc-key} gives an explicit canonicalization of the
stated equivalence relation. A fixed-center implementation can enumerate
48 centered transforms using its chosen origin convention. The canonical
construction provides a mathematical reference for these implementations.

\paragraph{Local representation and relaxed coordinates.}
Before canonicalization, the parent vacancy sites are expressed in the
consistent local Cartesian representation described in
Appendix~\ref{app:rev-coordinate-background}.

The integer key acts on parent vacancy sites. The one-to-four-vacancy
collection additionally compares relaxed full-AV geometries at
$0.001\,\text{\AA}$ tolerance (Appendix~\ref{app:rev-1to4-construction}).
Together, canonical vacancy patterns and full-geometry checks distinguish
configuration identity at their respective stages.

\paragraph{Transforming a complete event record.}
When changing the representative of a configuration, the same spatial
operation and atom correspondence must be applied to its entire record.
For atom permutation $P$, cubic operation $G$, and row translation $t$,
\begin{equation}
 \begin{gathered}
 X'=PXG^\top+\mathbf1t,\qquad
 S'_e=PS_eG^\top+\mathbf1t,\\
 (X_e^+)'=PX_e^+G^\top+\mathbf1t,\qquad
 D'_e=PD_eG^\top.
 \end{gathered}
 \label{eq:app-event-transform}
\end{equation}
All atom-indexed masks and any physical boundary data must transform
consistently. Because the EAM expression depends on interatomic distances,
these transformations preserve its energies and barriers when the complete
physical environment is transformed consistently, as follows from
Equation~\revEqref{eq:app-eam}. The neural transformation law is established
separately in Appendix~\ref{app:rev-equivariance}.

\subsection{Configuration deduplication, event matching, and data splitting}
\label{app:rev-dataset-event-distinction}

\paragraph{Configuration deduplication.}
The 48-operation rule groups cubic-symmetry-equivalent initial vacancy
environments into configuration classes. Each retained state then has
its own reference set of saddle mechanisms and connected products.

\paragraph{Event identity within a fixed state.}
The main-text reference matcher compares events from the same represented
$X$ using their full atomwise displacements and barriers:
\begin{equation}
 \max_a\|D_{e,a}-D_{e',a}\|_2\le0.05\,\text{\AA},
 \qquad
 |\Delta E_e-\Delta E_{e'}|\le0.02\,\mathrm{eV}.
 \label{eq:app-current-event-match}
\end{equation}
This is the event-matching criterion in Section~\ref{sec:experiments},
with no additional minimization over cubic operations. It is distinct from both initial-configuration
deduplication and the moving-union RMS used for judge labels. During
recovery evaluation the reference catalogue stays fixed, an output receives
at most one matching reference identity, and repeated hits to that identity
count once. Fixed reference identities define the recovery denominator,
providing a consistent basis for comparing outputs across methods.

Retaining spatial exit channels preserves their contribution to escape
rates. For an isolated vacancy, nearest-neighbor exchanges in different
directions can be related by cubic symmetry and have equal barriers.
They move the vacancy to different positions and contribute separately
to the total escape rate. Event-level evaluation retains these channels
and their multiplicities.

\paragraph{Grouping for the four-vacancy train/monitor/test split.}
All targets sharing an initial state belong to the same split. The reported
\textsc{BccFe-4vacAV-4000} split contains 3,700 training, 100 monitoring,
and 200 held-out states. The 200-state recovery evaluation uses that test
split; the fixed 50-state ablation and wall-time subset is described in
Appendix~\ref{app:rev-eval50}. Both evaluation populations are drawn from
the four-vacancy test bank.

State-level grouping keeps the reference events of each configuration
together. Parent-corpus construction and assignment to training, monitoring,
and test partitions are separate data-preparation stages.

For a split intended to separate symmetry-equivalence classes, all
symmetry-equivalent copies must remain in one group. If
$\mathcal K_{\mathrm{train}}$, $\mathcal K_{\mathrm{monitor}}$, and
$\mathcal K_{\mathrm{test}}$ are the configuration-key sets, the integrity
condition is
\begin{equation}
 \mathcal K_{\mathrm{train}}\cap\mathcal K_{\mathrm{monitor}}
 =\mathcal K_{\mathrm{train}}\cap\mathcal K_{\mathrm{test}}
 =\mathcal K_{\mathrm{monitor}}\cap\mathcal K_{\mathrm{test}}
 =\varnothing.
 \label{eq:app-split-integrity}
\end{equation}
This criterion defines symmetry-level split integrity and can be checked
directly from the released configuration-key sets.

\paragraph{One-to-four-vacancy release assignments.}
The separate \textsc{BccFe-1to4Vac} release manifest supplies
4,171/543/658 train/validation/test AV assignments. The generation record
reports that no parent ID or family ID is shared across these partitions.
These release assignments support future studies using the broader vacancy
collection; the reported experiments use the four-vacancy split. The single
one-vacancy AV supplies an elementary reference case within the release.

\section{Architecture and Training Details}
\label{app:rev-architecture}

\subsection{Encoder}
\label{app:rev-encoder}
The encoder uses the represented configuration $X$ defined in the main text, after the background coordinate preparation in Appendix~\ref{app:rev-periodic-background}. A directed edge $b\to a$ is present for a nonzero separation within $6\,\text{\AA}$, with at most 128 neighbors per atom. Its features are
\begin{equation}
 \begin{gathered}
 r_{ba}=X_a-X_b,\qquad d_{ba}=\|r_{ba}\|_2,\qquad
 \widehat r_{ba}=r_{ba}/d_{ba},\\
 e_{ba}=\left(\phi(d_{ba}),\ \bigoplus_{\ell=0}^{2}Y^{(\ell)}(\widehat r_{ba})\right),
 \end{gathered}
 \label{eq:rev-edge-features}
\end{equation}
where $\phi$ is a 128-dimensional Gaussian radial expansion and the spherical harmonics extend through degree two. Equation~\revEqref{eq:rev-edge-features} uses the prepared Cartesian edge vector directly for graph inclusion, distances, and angular features.

Species embeddings are augmented by an equivariant edge-degree embedding. Every atomic representation contains
\begin{equation}
 h_a=h_a^{(0)}\oplus h_a^{(1)}\oplus h_a^{(2)},\qquad
 h_a^{(\ell)}\in\mathbb R^{C_\ell\times(2\ell+1)},\qquad
 (C_0,C_1,C_2)=(128,64,32).
 \label{eq:rev-node-features}
\end{equation}
The last axis transforms under the degree-$\ell$ representation of $R\in\mathrm{SO}(3)$, giving $128+3(64)+5(32)=480$ components per atom. All benchmark atoms are Fe.

Suppressing the layer index, each of the six Equiformer blocks constructs an edge representation using equivariant source and destination projections and a radially conditioned depthwise tensor product \citep{liao2022equiformer}:
\begin{equation}
 t_{ba}=\operatorname{TP}_{\psi(\phi(d_{ba}))}
 \left(W_{\mathrm{src}}h_b+W_{\mathrm{dst}}h_a,\ Y(\widehat r_{ba})\right).
 \label{eq:rev-edge-tp}
\end{equation}
Here $\psi$ predicts invariant tensor-product weights. For each of four heads, scalar channels provide an invariant logit, while equivariant projection, gating, and a further separable tensor product provide the value $v_{ba}^{(h)}$:
\begin{equation}
 \begin{gathered}
 s_{ba}^{(h)}=w_h^\top\operatorname{act}(P_h^{(0)}t_{ba}),\qquad
 \alpha_{ba}^{(h)}=\operatorname{softmax}_{b\in\mathcal N(a)}s_{ba}^{(h)},\\
 m_a^{(h)}=\sum_{b\in\mathcal N(a)}\alpha_{ba}^{(h)}v_{ba}^{(h)}.
 \end{gathered}
 \label{eq:rev-encoder-attn}
\end{equation}
The heads are concatenated and equivariantly projected. Normalization, residual connections, and an equivariant feed-forward update complete the block. Attention-weight dropout is $0.2$ during training; Equation~\revEqref{eq:rev-encoder-attn} gives the pre-dropout weights. The output features used in the main text are $u_a\in\mathbb R^{128}$ and $V_a\in\mathbb R^{64\times3}$.

The equivariance result in Appendix~\ref{app:rev-equivariance} assumes atom-reindexing-compatible graph construction and neighbor truncation, including a compatible convention when the neighbor cap splits a distance tie.

\subsection{Slot refinement and cross-attention}
\label{app:rev-slot-details}
The main model has $M=96$, $d_s=256$, and $d_v=16$. We use the main-text stage index throughout: initialization is $t=0$, the two vector--scalar stages are $t=1,2$, the three scalar-only stages are $t=3,4,5$, and the final readout is $T=6$. The state entering initialization is the learned scalar seed $s_m^{(0)}=s_m$ and the zero vector state $Q_m^{(0)}=0$. Each slot acquires directional information through attention-weighted pooling of the encoded atomic vectors.

\paragraph{Initialization.}
The optional slot self-attention in Equation~\revEqref{eq:rev-slot-self} acts in place on the scalar seeds. It is enabled for the full-model setting and disabled for the cross-attention-only ablation in Table~\ref{tab:ablation}. With these scalar states, the initialization attention and vector readout are
\begin{equation}
 \begin{gathered}
 A^{(0,1)}_{ma}=\operatorname{softmax}_a\!
 \left[\frac{(W_qs_m^{(0)})^\top(W_ku_a^{(0)})}{\sqrt{d_s}}\right],\\
 U_m^{(0,1)}=\sum_aA^{(0,1)}_{ma}W_VV_a,\qquad
 s_m^{(1)}=s_m^{(0)},\qquad Q_m^{(1)}=U_m^{(0,1)}.
 \end{gathered}
 \label{eq:rev-slot-init}
\end{equation}
Here $u_a^{(0)}$ is an MLP embedding $128\to256\to256$ of $u_a$. The zero initial vector state and $\gamma^{(0,1)}=0$ reduce the main-text template to this scalar-attention bootstrap.

\paragraph{Two vector--scalar updates.}
For $t=1,2$, write $u_a^{(t)}=u_a^{(0)}$. The shared scalar maps are $W_q,W_k,W_z$, and the shared atomic vector maps are $W_K,W_V$. The slot-vector maps, scalar update MLPs, vector update maps, and compatibility coefficients are stage specific. Omitting only the single head index, the main-text attention becomes
\begin{equation}
 A_{ma}^{(t)}=\operatorname{softmax}_a\!\left[
 \frac{(W_qs_m^{(t)})^\top(W_ku_a^{(t)})}{\sqrt{d_s}}
 +\frac{\gamma^{(t)}}{\sqrt{3d_v}}
 \left\langle\widetilde W_Q^{(t)}Q_m^{(t)},W_KV_a\right\rangle_F
 \right].
 \label{eq:rev-slot-refine-attn}
\end{equation}
The maps $W_q,W_k,W_z,W_K,W_V$ are shared with initialization where applicable. Initial logits use the scalar term with zero vector-compatibility coefficient, while each update-stage $\gamma^{(t)}$ starts at one. The contexts and updates are
\begin{equation}
 \begin{aligned}
 z_m^{(t)}&=\sum_aA_{ma}^{(t)}W_zu_a^{(t)},&
 U_m^{(t)}&=\sum_aA_{ma}^{(t)}W_VV_a,\\
 s_m^{(t+1)}&=s_m^{(t)}+
 F^{(t)}\!\left(s_m^{(t)},z_m^{(t)},
 (\|U_{m,v}^{(t)}\|_2)_{v=1}^{d_v}\right),&
 Q_m^{(t+1)}&=Q_m^{(t)}+W_U^{(t)}U_m^{(t)}.
 \end{aligned}
 \label{eq:rev-slot-refine-update}
\end{equation}
The scalar update is a two-layer MLP. The invariant vector norms communicate geometric information to scalar fingerprints; vector-channel mixing preserves orientation. The scalar query/key/value maps in these stages and all vector-channel maps are bias free.

\paragraph{Scalar-only updates and readout.}
The three scalar-only stages use an independent embedding $u'_a\in\mathbb R^{256}$ from another $128\to256\to256$ MLP, with eight heads, pre-normalization, and the residual feed-forward update of Equation~\revEqref{eq:rev-slot-mha-update}. They leave $Q_m^{(t)}$ unchanged. In the attention notation of the main text, $\hat s_m=\mathrm N^{(t)}(s_m^{(t)})$ and $\hat u_a=\mathrm N^{(t)}(u_a^{(t)})$, with the appropriate separately parameterized normalizations; $\mathrm N^{(t)}$ is identity in the unnormalized stages and layer normalization in the scalar-only stages.

The final single-head readout at $T=6$ has its own $W_q,W_k,W_z,\widetilde W_Q,\widetilde W_V$. It recomputes attention with $\gamma^{(T)}=1$ and returns $A_{ma}^{(T)},z_m^{(T)},s_m^{(T)},Q_m^{(T)}$ for decoding, retaining the final slot states. Stage-specific vector maps mix channels identically across the three spatial components and use zero spatial bias.

\subsection{Decoder details}
\label{app:rev-decoder-details}
Omitting the final stage index as in the main text, the anchor is
\begin{equation}
 b_m=\sum_aA_{ma}X_a,\qquad \sum_aA_{ma}=1.
 \label{eq:rev-slot-anchor}
\end{equation}
The offset $X_a-b_m$ encodes each atom's position relative to the attention-weighted anchor. Computing this mean in the prepared local Cartesian environment gives the decoder a geometry-conditioned spatial reference.

The atom-derived generators are the 32 rows of $W_{\mathrm{atom}}V_a$, and the slot-derived generators are the eight rows of $W_{\mathrm{slot}}Q_m$. Higher-order encoder channels contribute through equivariant message passing. Together with the anchor-relative direction, they form the 41-row vector bank $\mathcal B_{ma}$ in Equation~\revEqref{eq:rev-bases}, combining three sources of geometric information for displacement synthesis.

The decoder uses the invariant conditioning
\begin{equation}
 \begin{gathered}
 \xi_{ma}=\left[s_m,z_m,u'_a,\chi_{ma},
          \log(1+\|X_a-b_m\|_2^2)\right],\\
 \chi_{ma}=\left(\frac{1}{\sqrt3}
    \left\langle(G_QQ_m)_g,(G_VV_a)_g\right\rangle\right)_{g=1}^8.
 \end{gathered}
 \label{eq:rev-decoder-features}
\end{equation}
The coordinates are numerically expressed in the fixed input length unit. With $s_m,z_m,u'_a\in\mathbb R^{256}$, $\xi_{ma}\in\mathbb R^{777}$. The maps $G_Q$ and $G_V$ have channel widths $16\to8$ and $64\to8$, respectively.

The four MLPs follow Equation~\revEqref{eq:rev-decoder-coef}. Gate and scale networks $g_{\mathrm{act}},g_{\mathrm{mag}}$ have widths $777\to256\to1$; coefficient networks $f_{\mathrm{coef}},f_{\mathrm{res}}$ have widths $777\to256\to256\to41$. Hidden activations are SiLU, and the outputs are $p_{ma}$, $\kappa_{ma}$, $q_{ma}$, and $q_{ma}^{\mathrm{res}}$.

The output layers of $g_{\mathrm{act}},g_{\mathrm{mag}},f_{\mathrm{res}}$ start with zero weights and biases $-2,-3.9,0$, respectively. The learned residual scale $\gamma_{\mathrm{res}}=\sigma(\vartheta)$ starts at $0.05$. Initial participation is $\sigma(-2)\approx0.12$, the magnitude parameter is $\operatorname{softplus}(-3.9)\approx0.02\,\text{\AA}$, and the residual output is zero. The realized displacement magnitude also depends on the directional combination.

The residual lies outside the activity gate in Equation~\revEqref{eq:rev-decoder}, allowing small-participation atoms to express surrounding relaxation. Radial normalization uses $v/\max(\|v\|_2,\epsilon)$ with $\epsilon=10^{-8}$, giving norm $\min(1,\|v\|_2/\epsilon)$ and zero output at $v=0$. This stabilized geometric operation preserves the rotation law used in the equivariance proof.

\section{Displacement Supervision and Active-Aware Matching}
\label{app:rev-training}

\subsection{Displacement fields and the uniform objective}
For each initial configuration $X\in\mathbb R^{N\times3}$, the predictor
produces $M$ displacement fields $\widehat D_m$, supervised by $K$ reference
fields $D_e=S_e-X$. All fields share atom correspondence and Cartesian units
of $\text{\AA}$. Training uses the full residual $\widehat D_m-D_e$ directly;
inference takes the initial configuration as input.

The scalar SmoothL1 penalty and its all-atom reduction are
\begin{equation}
 \rho_\beta(t)=
 \begin{cases}
 t^2/(2\beta),& |t|<\beta,\\
 |t|-\beta/2,& |t|\ge\beta,
 \end{cases}
 \qquad
 \ell_\beta(E)=\frac1N\sum_{a=1}^N\sum_{d=1}^3\rho_\beta(E_{ad}).
 \label{eq:rev-smoothl1}
\end{equation}
The objective sums the three coordinate penalties and averages over all
$N$ atoms, including those with zero reference displacement. The main
Uniform model uses $C_{me}=\ell_\beta(\widehat D_m-D_e)$ for both assignment
and matched regression, with $\beta=0.03\,\text{\AA}$.

\subsection{The active-aware ablation}
\label{app:rev-active}
The active-aware ablation changes both the proposal--reference assignment
cost and matched displacement loss. For each reference event, define
\begin{equation}
 \begin{gathered}
 g_{e,a}=\|D_{e,a}\|_2,\qquad
 a_{ma}=\|\widehat D_{m,a}\|_2,\qquad
 a_e^*=\operatorname*{arg\,max}_a g_{e,a},\\
 \mathcal A_e=
 \begin{cases}
 \{a:g_{e,a}\ge\delta_{\mathrm{act}}\},
     &\text{if this set is nonempty},\\
 \{a_e^*\},&\text{otherwise},
 \end{cases}
 \qquad \delta_{\mathrm{act}}=0.05\,\text{\AA}.
 \end{gathered}
 \label{eq:app-active-set}
\end{equation}
The active set is defined from each event's reference displacement.
The fallback ensures a nonempty set, with tied maxima resolved by the first
maximum in stored atom order.

Let the coordinatewise penalty at one atom be
\[
 E_{me,a}=\sum_{d=1}^3
     \rho_\beta(\widehat D_{m,ad}-D_{e,ad}).
\]
The five enabled components are
\begin{equation}
 \begin{aligned}
 \ell_{\mathrm{all},me}
   &=\frac1N\sum_a E_{me,a},
 &\ell_{A,me}
   &=\frac1{|\mathcal A_e|}\sum_{a\in\mathcal A_e}E_{me,a},\\
 \ell_{\max,me}&=E_{me,a_e^*},
 &\ell_{\mathrm{mag},me}
   &=\frac1{|\mathcal A_e|}\sum_{a\in\mathcal A_e}
          \rho_\beta(a_{ma}-g_{e,a}),\\
 \ell_{\cos,me}
   &=\frac1{|\mathcal A_e|}\sum_{a\in\mathcal A_e}
       \omega_{m,a}\bigl[1-
        \cos_{\epsilon_{\cos}}(\widehat D_{m,a},D_{e,a})\bigr].
 \end{aligned}
 \label{eq:rev-active-components}
\end{equation}
Here the cosine and its detached amplitude gate are
\begin{equation}
 \begin{split}
 \cos_{\epsilon_{\cos}}(u,v)
 &=\frac{\langle u,v\rangle}
 {\max(\|u\|_2,\epsilon_{\cos})
  \max(\|v\|_2,\epsilon_{\cos})},
 \qquad\epsilon_{\cos}=10^{-4},\\
 \omega_{m,a}
 &=\operatorname{stopgrad}\!\left[
       \operatorname{clip}\!\left(
       \frac{a_{ma}}{\max(\delta_{\mathrm{act}},10^{-6})},0,1\right)
       \right].
 \end{split}
 \label{eq:app-active-direction-gate}
\end{equation}
The stabilizers use the supplied length convention. The detached gate
reduces directional supervision near a zero prediction. Each active mean
is normalized by its event's active-atom count, and the maximum-atom term
evaluates the atom with the largest reference displacement.

The enabled pair objective is
\begin{equation}
 C^{\mathrm{active}}_{me}
 =\ell_{\mathrm{active},me}
 =0.25\ell_{\mathrm{all},me}
  +\ell_{A,me}+\ell_{\max,me}
  +0.5\ell_{\mathrm{mag},me}+0.1\ell_{\cos,me}.
 \label{eq:rev-active-loss}
\end{equation}
The all-atom term remains present at reduced weight. The active variant
retains $\beta=0.03\,\text{\AA}$ and the same sequential matching and
configuration-level aggregation as Uniform. The
$0.05\,\text{\AA}$ training activity cutoff is distinct from the judge's
$0.3\,\text{\AA}$ moving-atom cutoff and its score threshold $\tau$.

\subsection{Sequential assignment and loss aggregation}
Here $K=|\mathcal E_X^{\mathrm{ref}}|$ counts reference targets, following
the training notation of Section~\ref{sec:rev-loss}.
Let $\mathcal O_r$ denote the slots available at round $r$, starting with
$\mathcal O_1=\{1,\ldots,M\}$. For a nonempty reference set, define the
rectangular one-to-one assignment family
\begin{equation}
 \begin{split}
 \mathfrak M(\mathcal O_r,[K])
 =\{\Pi\subseteq\mathcal O_r\times[K]:\;&
       |\Pi|=\min(|\mathcal O_r|,K),\\
 &\deg_\Pi(m)\le1,\quad\deg_\Pi(e)\le1\}.
 \end{split}
 \label{eq:rev-match-feasible}
\end{equation}
At each round, Hungarian matching solves
\begin{equation}
 \Pi_r\in\operatorname*{arg\,min}_{\Pi\in\mathfrak M(\mathcal O_r,[K])}
       \sum_{(m,e)\in\Pi}C_{me},\qquad
 \mathcal O_{r+1}=\mathcal O_r\setminus
       \{m:(m,e)\in\Pi_r\text{ for some }e\}.
 \label{eq:app-sequential-assignment}
\end{equation}
Use $C_{me}=\ell_{\mathrm{all},me}$ for Uniform and
$C_{me}=C^{\mathrm{active}}_{me}$ for the active-aware ablation. The detached
cost matrix determines the discrete assignment; regression then propagates
gradients through the predicted fields of the selected pairs.

The procedure performs at most $R_{\mathrm{match}}=4$ rounds, stopping
when no slots remain. The full reference set is available again at every
round. A slot is used at most once, whereas each target can receive one
slot per round. With finite pair costs and $K>0$, the total number of
matched pairs is
\begin{equation}
 \sum_r|\Pi_r|=\min(M,R_{\mathrm{match}}K).
 \label{eq:app-number-matches}
\end{equation}
The number of supervised slots adapts to the reference-set size: for
example, $K=2$ gives eight matched slots when $M=96$. When $M\ge K$, the
first round covers every target, and subsequent rounds assign the remaining
slots against the full reference set. Each round solves its own residual
assignment problem.

\paragraph{Primary, duplicate, and batch normalization.}
Write $\Pi_+=\bigcup_{r\ge2}\Pi_r$. Using the chosen pair objective
$C_{me}$, define
\begin{equation}
 L_1=\frac1{|\Pi_1|}\sum_{(m,e)\in\Pi_1}C_{me},\qquad
 L_+=
 \begin{cases}
 |\Pi_+|^{-1}\displaystyle\sum_{(m,e)\in\Pi_+}C_{me},&|\Pi_+|>0,\\
 0,&|\Pi_+|=0.
 \end{cases}
 \label{eq:app-primary-duplicate-means}
\end{equation}
The configuration loss uses the main-text notation $\mathcal L_x$ and aggregation
\begin{equation}
 \mathcal L_x=\lambda_D(L_1+\lambda_+L_+),
 \qquad\lambda_D=10,\quad\lambda_+=0.5.
 \label{eq:app-config-displacement-loss}
\end{equation}
The primary and duplicate terms are averaged separately before applying
the $0.5$ duplicate weight. All later-round pairs contribute to the common
duplicate mean. With one matching round or an empty duplicate set, $L_+=0$;
an empty reference set gives zero loss. Batch reduction weights each
configuration equally.

Repeated supervision gives a reference event multiple assigned proposals
across rounds. Assigned pairs provide displacement gradients through the
shared predictor, while distinct event recovery is measured after Dimer
refinement and reference matching.

\subsection{Predictor optimization and checkpoint scope}
The predictor training recipe uses AdamW with learning rate
$8\times10^{-5}$, weight decay $10^{-5}$, batch size two, gradient clipping
at five, and a 120-epoch cosine schedule. The data split comprises 3,700
training, 100 monitoring, and 200 held-out configurations. Monitoring-set
first-round displacement loss selects the predictor checkpoint. Downstream
recovery and search-force metrics evaluate its physical-search performance.
Appendix~\ref{app:rev-judge} specifies the frozen Full-model checkpoint and
monitoring procedure for the judge.

\section{Initial-State-Only Neural RMS Judge}
\label{app:rev-judge}

\subsection{Frozen predictor and inference inputs}
The neural judge scores proposals using only the initial configuration,
its frozen encoder features, and the complete predicted displacement
pool. For the Full-model pairing described here, the predictor checkpoint is
epoch 103, loaded with strict parameter matching and kept in evaluation
mode. Its scalar features $u_a\in\mathbb R^{128}$, vector channels
$V_a\in\mathbb R^{64\times3}$, and proposals $\widehat D_m$ are detached
from judge optimization. This frozen-feature readout scores the proposal
pool without force evaluations. Reference displacements supply offline
training and monitoring labels; inference uses the initial state and
its predicted proposals.

\subsection{Atom features, pooling, and network widths}
Let $a_{ma}=\|\widehat D_{m,a}\|_2$. To avoid confusing the judge feature
with the decoder coefficients $q_{ma}$, denote it by $\zeta_{ma}$:
\begin{equation}
 \zeta_{ma}=\left[
 u_a,\;
 \bigl(\langle\widehat D_{m,a},V_{a,v}\rangle\bigr)_{v=1}^{64},\;
 a_{ma},\;\log(1+a_{ma})\right]\in\mathbb R^{194}.
 \label{eq:rev-judge-feat}
\end{equation}
Numerical amplitudes in logarithms are expressed in the fixed
$\text{\AA}$ convention. A shared atom network produces
$h_{ma}^{J}=\phi_J(\zeta_{ma})\in\mathbb R^{128}$, where
\[
 \phi_J:
 194\xrightarrow{\mathrm{Linear}}192
 \xrightarrow{\mathrm{SiLU},\;\mathrm{Dropout}(0.2)}192
 \xrightarrow{\mathrm{Linear}}128
 \xrightarrow{\mathrm{SiLU}}128.
\]
A learned affine map $w_J^\top h+b_J$ gives the attention score. The
complete readout is
\begin{equation}
 \begin{gathered}
 \alpha^J_{ma}=\operatorname{softmax}_{a}
   \bigl[w_J^\top h^J_{ma}+b_J+3\log(1+a_{ma})\bigr],\\
 v_m^J=\left[\sum_a\alpha^J_{ma}h^J_{ma},\;
                    \max_a h^J_{ma},\;g_m\right]\in\mathbb R^{266},\\
 t_m^J=\psi_J(v_m^J),\qquad \widehat p_m=\sigma(t_m^J),
 \end{gathered}
 \label{eq:rev-judge-score}
\end{equation}
where the maximum is channelwise and $g_m\in\mathbb R^{10}$ is defined
below. The output network is
\[
 \psi_J:
 266\xrightarrow{\mathrm{Linear}}128
 \xrightarrow{\mathrm{SiLU},\;\mathrm{Dropout}(0.2)}128
 \xrightarrow{\mathrm{Linear}}1.
\]
Both dropout layers are disabled for scoring. The head's sigmoid output
provides the score used for threshold selection.

\subsection{The ten geometric descriptors, in implementation order}
All descriptors are computed \emph{before} filtering, from the full
proposal pool of the configuration. Define
\[
 a_m^{\max}=\max_a a_{ma},\qquad
 j_m^*=\operatorname*{arg\,max}_a a_{ma},\qquad
 \mathcal U_m=\{a:a_{ma}\ge0.3\,\text{\AA}\}.
\]
The descriptor vector is
\[
 g_m=(f_{\mathrm{mmax}},f_{\mathrm{band}},f_{\mathrm{njump}},
 f_{\mathrm{nhalo}},f_{\mathrm{nclust}},f_{\mathrm{second}},
 f_{\mathrm{sum}},f_{\mathrm{red}},f_{\mathrm{rank}},f_{\mathrm{coher}})_m.
\]
The following definitions specify the raw features supplied to the head;
no feature standardization is inserted.

\paragraph{Amplitude and atom counts (features 1--4 and 7).}
\begin{equation}
 \begin{aligned}
 f_{\mathrm{mmax},m}&=a_m^{\max},\\
 f_{\mathrm{band},m}&=\max\{0,\;0.8\,\text{\AA}-a_m^{\max},\;
                                  a_m^{\max}-1.6\,\text{\AA}\},\\
 f_{\mathrm{njump},m}&=\bigl|\{a:a_{ma}\ge0.5\,\text{\AA}\}\bigr|,\\
 f_{\mathrm{nhalo},m}&=\bigl|\{a:0.05\,\text{\AA}\le a_{ma}
                                      <0.5\,\text{\AA}\}\bigr|,\\
 f_{\mathrm{sum},m}&=\sum_a a_{ma}.
 \end{aligned}
 \label{eq:app-judge-amplitude-features}
\end{equation}
The count features retain absolute atom counts, and $f_{\mathrm{sum},m}$ sums their displacement amplitudes.

\paragraph{Moving-atom clusters (features 5--6).}
Construct an undirected graph on $\mathcal U_m$, connecting two moving
atoms when their \emph{initial} positions are separated by at most
$3\,\text{\AA}$. Let $\mathcal C_{m,1},\ldots,\mathcal C_{m,c_m}$ be its
connected components and let
$P_{m,j}=\max_{a\in\mathcal C_{m,j}}a_{ma}$ be their peak amplitudes.
Then $f_{\mathrm{nclust},m}=c_m$ and $f_{\mathrm{second},m}$ is the
second-largest component peak. An empty moving set gives $(0,0)$, and
fewer than two components gives $f_{\mathrm{second},m}=0$.
The epoch-103 judge computes distances as $\|X_a-X_b\|_2$ in the
prepared nonperiodic local-coordinate representation of the active volume.

\paragraph{Pool-relative redundancy and rank (features 8--9).}
Let
$\mathcal H_m=\{n\ne m:a_n^{\max}>a_m^{\max}\}$.
The redundancy distance is
\begin{equation}
 f_{\mathrm{red},m}=
 \begin{cases}
 \displaystyle\min_{n\in\mathcal H_m}
 \left(\frac1N\sum_a
       \|\widehat D_{m,a}-\widehat D_{n,a}\|_2^2\right)^{1/2},
       &\mathcal H_m\ne\varnothing,\\
 5\,\text{\AA},&\mathcal H_m=\varnothing.
 \end{cases}
 \label{eq:app-judge-redundancy}
\end{equation}
This feature measures all-atom Cartesian RMS separation within the
proposal pool. The strict $>$ test selects higher-amplitude proposals.
Rank is the zero-based position in the descending sort of $a_m^{\max}$,
ranging from $0$ to $M-1$, with ties following the array-sort ordering.

\paragraph{Directional coherence (feature 10).}
When $|\mathcal U_m|>1$ and $a_m^{\max}>0$,
\begin{equation}
 f_{\mathrm{coher},m}=
 \frac1{|\mathcal U_m|}\sum_{a\in\mathcal U_m}
 \left\langle
  \frac{\widehat D_{m,a}}{\max(a_{ma},10^{-8})},\;
  \frac{\widehat D_{m,j_m^*}}{a_m^{\max}}
 \right\rangle;
 \label{eq:app-judge-coherence}
\end{equation}
otherwise it equals one. The mean includes the maximum-displacement atom
itself. $j_m^*$ is the first maximum in the stored atom order.

\paragraph{Geometric invariance and tie conventions.}
Norms, dot products, local distances, and atom pooling preserve rotations
under the predictor's transformation laws. Exact atom-permutation invariance
also requires compatible handling of tied maximum amplitudes in coherence;
proposal-rank ties follow pool order. The redundancy and rank features
provide context from the complete proposal pool, so scores reflect both
individual geometry and the alternatives proposed for the same state.

\subsection{Offline labels and judge training}
For each proposal--reference pair, use the main-text moving union
\[
 U_{me}=\{a:\max(\|\widehat D_{m,a}\|_2,\|D_{e,a}\|_2)
                           \ge0.3\,\text{\AA}\}.
\]
For nonempty unions, compute
\begin{equation}
 r^U_{me}=\left(\frac1{|U_{me}|}\sum_{a\in U_{me}}
         \|\widehat D_{m,a}-D_{e,a}\|_2^2\right)^{1/2},\qquad
 d^U_{me}=\max_{a\in U_{me}}
         \|\widehat D_{m,a}-D_{e,a}\|_2.
 \label{eq:app-judge-label-errors}
\end{equation}
Empty unions are ineligible, equivalently having infinite RMS. First
select the reference $e_m^*$ of smallest $r^U_{me}$, keeping the first
reference in catalogue order at a tie. Then set
\begin{equation}
 y_m=\mathbf1\{r^U_{m e_m^*}<0.35\,\text{\AA}
              \text{ and }d^U_{m e_m^*}<0.5\,\text{\AA}\}.
 \label{eq:app-judge-binary-label}
\end{equation}
If no eligible reference exists, $y_m=0$. Both tests are strict and apply
to the RMS-nearest reference selected above. These geometric-accuracy
labels train proposal screening; the final event evaluation uses the
all-atom displacement and barrier criteria.

\paragraph{Head training and checkpoint selection.}
Gradient updates use training configurations only, as stated in
Section~\ref{sec:rev-judge}. The stage-1 head-training recipe uses the
first 200 configurations of \texttt{train\_strict} for those updates
and the first 50 monitoring configurations for head selection; no
held-out test configurations enter either role. One optimizer step processes all slots of one configuration,
and the binary cross-entropy is averaged over those slots.

Let $\pi_+$ be the positive-label fraction obtained by concatenating all
training labels. The positive-class weight is
\begin{equation}
 w_+=\frac{1-\pi_+}{\max(\pi_+,0.001)},\qquad 0<\pi_+<1.
 \label{eq:app-judge-class-weight}
\end{equation}
For one configuration, the logit loss is
\[
 \mathcal L_J=-\frac1M\sum_m
 \left[w_+y_m\log\sigma(t_m^J)
 +(1-y_m)\log(1-\sigma(t_m^J))\right].
\]
The implementation uses binary cross-entropy with logits for numerical
stability. The head uses AdamW with learning rate $8\times10^{-4}$ and
weight decay $10^{-4}$. Torch and Python random seeds are reset to zero,
and training configurations are shuffled each epoch. Training lasts at
most 40 epochs. After each epoch, monitoring ROC AUC is computed from the
concatenated proposal scores and labels. Strictly improved AUC saves the
head; six consecutive non-improving epochs stop training. The best saved
head is restored with strict parameter matching and placed in evaluation
mode.

\subsection{Deployment and proposal identifiers}
Compute the ten descriptors on the complete proposal pool, evaluate the
head, and refine exactly the indices
\begin{equation}
 I_X(\tau)=\{m:\widehat p_m\ge\tau\}.
 \label{eq:app-judge-selection}
\end{equation}
Selection retains the original slot/search identifiers and uses descriptors
computed on the full pool. The default score threshold $\tau=10^{-5}$ is
used in the 50-configuration wall-time comparison
(Figure~\ref{fig:rev-walltime}). The 200-configuration recovery curves vary
this threshold, while the ablations in Table~\ref{tab:ablation} disable
selection. The sigmoid readout is used as a geometric proposal-selection
score, with physical recovery assessed after refinement.

\section{Native Dimer Refinement and Validation}
\label{app:rev-dimer}
\label{app:rev-refinement}

\subsection{Enabled path and coordinate convention}
The physical search uses the native SEAKMC Dimer implementation on the
fixed interatomic potential. The geometric background of Dimer is given
in Appendix~\ref{app:rev-dimer-background}; here we specify the supplied
numerical path. Searches use Cartesian active-atom coordinates with the
native fixed/buffer surroundings. The options \texttt{SNC},
\texttt{ApplyMass}, \texttt{SearchBuffer}, \texttt{LocalRefine}, and
\texttt{EarlyStop} are disabled, while \texttt{LocalRelax} is enabled.
The main Dimer search refines every submitted proposal;
\texttt{LocalRefine=False} specifies that its additional local-refinement
branch is disabled.

The paper writes atomic fields as $N\times3$ arrays. The native search
stores their transpose as $3\times N$ arrays. Below, $y$ and $n$ denote
the vectorized active-coordinate center and orientation, respectively,
and $f(y)=-\nabla V(y)$ denotes the physical force in that same coordinate
space. The potential includes fixed surroundings, while $n$ spans the
active coordinates. Native magnitude operations use
$\operatorname{mag}_{\mathrm{native}}$, corresponding to the
\texttt{mat\_mag} helper and its stopping convention.

\subsection{Random axes, proposal loading, and orientation reset}
For a random-start search, the native initial-axis sampler draws an array
with independent entries from $[-0.5,0.5)$. The wrapper sets
\begin{equation}
 \operatorname{seed}(c,i)
 =(2026091400+10000c+i)\bmod 2^{32},
 \label{eq:app-dimer-random-seed}
\end{equation}
before the native draw, where $c$ is the configuration ordinal and $i$
is the search identifier. Angle checking against previous axes follows
the native settings, completing the initial-axis sampling protocol.

For a learned or transferred field $D\in\mathbb R^{N\times3}$, the
bridge transposes the full field for preload and forms a separate
translation-removed direction. In paper notation,
\begin{equation}
 \bar D=\frac1N\sum_a D_a,\qquad
 \widetilde n(D)=
 \frac{\operatorname{vec}(D-\mathbf1\bar D)}
      {\|D-\mathbf1\bar D\|_F}.
 \label{eq:app-dimer-preload-axis}
\end{equation}
The helper accepts finite fields whose translation-removed norm exceeds
$10^{-12}$. Mean removal defines the axis; the complete field $D$ preloads
the candidate center $X+D$.

Native loading pads or restricts the field to the configured active
coordinates and applies any type/fixed-coordinate scaling. Write its
actual center field as $D^{\mathrm{exec}}$, and let $n_{\mathrm{pre}}$
be the native axis immediately before loading. The bridge supplies
$\widetilde n(D)$; any intervening native axis handling remains governed
by the job settings. Let $\mathsf S$ denote the native type-scaling
operator on the axis, or the identity when no such scaling is enabled.
The loading/reset rule is
\begin{equation}
 n_0=
 \begin{cases}
 \operatorname{unit}\!\left(
 \mathsf S\bigl[r_{\mathrm{reset}}n_{\mathrm{pre}}
                  +\operatorname{vec}(D^{\mathrm{exec}})\bigr]\right),
       &\texttt{ResetVN04Preload=True},\\
 \operatorname{unit}(\mathsf S n_{\mathrm{pre}}),&\text{otherwise}.
 \end{cases}
 \label{eq:app-dimer-axis-reset}
\end{equation}
The configured coefficient is $r_{\mathrm{reset}}=0.2$ in the
implementation's numerical $\text{\AA}$ convention. In the unscaled
case with no intervening axis change, this reduces to normalizing
$0.2\widetilde n(D)+\operatorname{vec}(D)$. This initializes the orientation
from the proposal. Subsequent centering, rescaling, and angle controls
follow the resolved job settings.

On an admissible, already represented displacement field $D$, the
complete initialization in the conditional analysis is written
\begin{equation}
 \mathcal I_X(D)=\bigl(X+D,\;n_0(D)\bigr).
 \label{eq:rev-initializer}
\end{equation}
Here $n_0(D)$ includes the configured orientation processing. The admissible
domain $\mathcal D_X$ enforces coordinate constraints and a finite, nonzero
processed orientation. The recovery analysis applies to the executed field
and this configured initialization through its loading condition and
local-success assumption.

\subsection{Curvature estimation, rotation, and force projection}
The native implementation reuses the center force to estimate endpoint
curvature. At center $y$, let $f_0=f(y)$ and $f_1=f(y+h_{\mathrm{dim}}n)$.
With \texttt{LocalRefine=False}, it computes
\begin{equation}
 f_2=2f_0-f_1,\qquad
 c=\frac{\langle f_2-f_1,n\rangle}{2h_{\mathrm{dim}}},\qquad
 f_{\mathrm{rot}}
 =\frac{(f_1-f_2)-n\langle f_1-f_2,n\rangle}
        {2h_{\mathrm{dim}}}.
 \label{eq:app-native-curvature}
\end{equation}
The reused-force estimate $f_2$ supplies the opposite-endpoint approximation.
The curvature $c$ and tangential force drive the native orientation update,
with separation $h_{\mathrm{dim}}=0.005\,\text{\AA}$.

\paragraph{Native rotation and force projection.}
The rotation update maintains a conjugate direction, with a restart based
on the current and previous rotation forces. Its coefficient has the form
\begin{equation}
 \gamma(f,f_{\mathrm{old}})=
 \begin{cases}
 \displaystyle\frac{\langle f,f-f_{\mathrm{old}}\rangle}
                      {\langle f_{\mathrm{old}},f_{\mathrm{old}}\rangle},
 &\begin{gathered}
 |\langle f,f_{\mathrm{old}}\rangle|
       \le\tfrac12\langle f_{\mathrm{old}},f_{\mathrm{old}}\rangle,\\[-2pt]
 \langle f_{\mathrm{old}},f_{\mathrm{old}}\rangle\ne0,
 \end{gathered}\\
 0,&\text{otherwise}.
 \end{cases}
 \label{eq:app-native-conjugate-coefficient}
\end{equation}
For rotation, the accumulated direction is normalized to a search vector
$u$. A forward angular sample updates the orientation and search vector
by
\[
 (n,u)\longmapsto
 \bigl(n\cos\theta+u\sin\theta,\;
       u\cos\theta-n\sin\theta\bigr),\qquad\theta=\pi/4.
\]
Writing the directional rotation-force samples as $r_1,r_2$, the return
stage applies the incremental angle
\begin{equation}
 \theta_{\mathrm{inc}}=
 \begin{cases}
 -\tfrac12\arctan(r_1/r_2),&r_2\ne0,\\
 -\tfrac12\pi,&r_2=0
 \end{cases}
 +\tfrac12\pi\,\mathbf1\{r_2>0\}-\tfrac14\pi.
 \label{eq:app-native-rotation-angle}
\end{equation}
Native normalization and enabled axis controls follow this update.
Rotation-force tests and the rotation counter use
\texttt{FMin4Rot}, \texttt{FThres4Rot}, and \texttt{NMax4Rot}.

For translation, write $f_{\parallel}=n\langle f_0,n\rangle$.
The effective force is curvature dependent:
\begin{equation}
 f_{\mathrm{eff}}=
 \begin{cases}
 f_0-2f_{\parallel},&c<0,\\
 -f_{\parallel},&c\ge0\text{ and the horizon mode is off},\\
 f_0-f_{\parallel},&c\ge0\text{ and the horizon mode is on}.
 \end{cases}
 \label{eq:app-native-projected-force}
\end{equation}
The horizon switch follows the native curvature, energy, and initial-well
logic, selecting the appropriate positive-curvature branch.

\subsection{Two-stage translation and fine-step controls}
The translation direction is also updated by a restarted conjugate-force
rule using Equation~\revEqref{eq:app-native-conjugate-coefficient} and is
normalized to $u$. Let $t_f$ and $t_m$ be the current trial and maximal
translation settings. The forward stage moves the current center by
$t_fu$. With effective forces before and after that trial, define
\[
 a_1=\langle f_{\mathrm{before}},u\rangle,\qquad
 a_2=\langle f_{\mathrm{trial}},u\rangle,\qquad
 c_t=\frac{a_1-a_2}{t_f}.
\]
The second stage adds a further displacement $t_cu$. Its native rule is
\begin{equation}
 \begin{gathered}
 t_* = \frac{(a_1+a_2)/2}{c_t+10^{-12}}
       \quad\text{when }c_t\ge0,\\
 t_c=
 \begin{cases}
 t_m,&c_t<0,\\
 t_m-t_f,&c_t\ge0,\ |t_*|>t_m,\ t_*\ge0,\\
 t_f-t_m,&c_t\ge0,\ |t_*|>t_m,\ t_*<0,\\
 t_*-t_f/2,&c_t\ge0,\ |t_*|\le t_m.
 \end{cases}
 \end{gathered}
 \label{eq:app-native-translation}
\end{equation}
The combined center displacement is $(t_f+t_c)u$, with $t_c$ determined by
the trial forces and the configured maximal-step parameter in this
two-stage update.

A preloaded search uses the native \texttt{FineStep} settings. The
supplied ordinary trial/maximal defaults are
$0.015/0.05\,\text{\AA}$, while the fine-step defaults are
$0.003/0.01\,\text{\AA}$. When the decay style is not fixed, the source
updates them in the form
\[
 t_f(j)=t_f^{(0)}d(j)+t_{\min},\qquad
 t_m(j)=t_m^{(0)}d(j)+t_{\min}r_{\min},
\]
with the corresponding ordinary or fine-step parameter family.
The decay factor is either
$d(j)=d_0^{\lfloor j/J\rfloor}$ or $d_0^{j/J}$ under the configured
style. The resolved job settings select the ordinary or fine-step family
and its decay parameters for each search.

\begin{table}[t]
 \centering\small
 \caption{\textbf{Native numerical controls and defaults.}
 Defaults identify the implementation's parameter families. Each experiment
 uses its resolved settings and the reported translation cap.}
 \label{tab:app-native-settings}
 \begin{tabular}{@{}p{0.38\linewidth}p{0.22\linewidth}p{0.32\linewidth}@{}}
 \toprule
 Quantity / native key & Value & Scope \\
 \midrule
 \texttt{DimerSep} & $0.005\,\text{\AA}$ & supplied setting/default \\
 Ordinary trial / maximal step & $0.015/0.05\,\text{\AA}$ & Random defaults \\
 Fine trial / maximal step & $0.003/0.01\,\text{\AA}$ & preload defaults \\
 \texttt{FConv} & $10^{-6}$ & native force-magnitude threshold \\
 \texttt{FMin4Rot} / \texttt{FThres4Rot} & $0.001/0.1$ & native rotation tests \\
 \texttt{NMax4Rot} & $3$ & native rotation counter setting \\
 \texttt{Tol4Connect} & $0.1\,\text{\AA}$ & maximum coordinate component \\
 \texttt{NMax4Trans} & $300$ & ML and event reuse \\
 \texttt{NMax4Trans} & $450$ & Random, up to 80 starts \\
 \bottomrule
 \end{tabular}
\end{table}

\subsection{Search loop, stopping, and cap handling}
The loop calls \texttt{dimer\_force} before the first preload to establish
the initial energy. At the first iteration, an available preload is then
loaded, the runner is prepared again, and another
\texttt{dimer\_force} call evaluates the preloaded configuration.
Subsequent iterations alternate the native rotation and translation
states. Rotation completion advances the native \texttt{NTSITR} counter,
which supplies the recorded translation count.

In the enabled branch, the force-convergence flag \texttt{FO\_CON} uses
\[
 \operatorname{mag}_{\mathrm{native}}(\texttt{FC})
       <\texttt{FConv},
\]
after the appropriate projection/update, and \texttt{ST\_CON} requires
\texttt{NTSITR > IgnoreTransSteps}. The search converges when both flags
hold. The enabled branch uses these force and step flags for convergence;
energy convergence is recorded separately. Native relaxation and
connectivity checks complete the acceptance protocol.

The cap check is applied when
\texttt{NTSITR >= NMax4Trans}; it sets the maximum-translation flag and
terminates the search. Cap-reaching outputs are invalidated in native
finalization. The downstream scorer also requires a translation count
strictly below the configured cap. Optional native displacement, barrier,
and angle checks remain governed by the job configuration.

\subsection{Native two-sided relaxation and connectivity}
With \texttt{LocalRelax} enabled and the additional refinement/early-stop
branches disabled, the two relaxation seeds are constructed from the
last search segment. In row-coordinate notation, let $Y_c$ denote
\texttt{XCR}, $Y_0$ denote \texttt{X0}, and $r_{\mathrm{rel}}$ denote
\texttt{DRATIO}. The seeds are
\begin{equation}
 Z^{(0)}_+=Y_0+r_{\mathrm{rel}}(Y_c-Y_0),\qquad
 Z^{(0)}_-=Y_0-r_{\mathrm{rel}}(Y_c-Y_0).
 \label{eq:app-native-downhill-seeds}
\end{equation}
Each seed is relaxed by the native \texttt{SPSRELAX} runner with the
configured fixed surroundings. Relaxed active coordinates are shifted
back relative to the initial configuration by the native periodic
reboxing convention. Thus the final search segment provides both seeds
for the native two-sided connectivity test.


The implementation first identifies the farther and nearer relaxed
branches using their overall native displacement magnitude from $X$.
For those selected branches, it then computes
\[
 d_{\mathrm{far}}=\max_{a,d}|(Z_{\mathrm{far}}-X)_{ad}|,\qquad
 d_{\mathrm{near}}=\max_{a,d}|(Z_{\mathrm{near}}-X)_{ad}|.
\]
If $d_{\mathrm{far}}<\texttt{Tol4Connect}$, it marks the event invalid
and disconnected. Otherwise the event is connected when
$d_{\mathrm{near}}<\texttt{Tol4Connect}$; when this latter test fails,
the relaxation routine can return a valid but disconnected result.
Final native barrier and displacement cuts can additionally invalidate
an output. Connectivity uses a maximum coordinate-component tolerance;
reference matching uses the all-atom Euclidean displacement gate.

\subsection{Returned saddle geometry and the numerical map}
The wrapper returns the final active saddle coordinates from
\texttt{XCR.T}, the recorded barrier \texttt{BARR}, native valid and
connected flags, the cap flag, translation count, and search-force
count. The saddle field for scoring is
\begin{equation}
 S_{\mathrm{out}}=\texttt{saddle\_active\_coords},\qquad
 D_{\mathrm{out}}=S_{\mathrm{out}}-X.
 \label{eq:app-native-recorded-saddle}
\end{equation}
Scoring uses this saved saddle geometry and its associated barrier.
Native \texttt{XDISP} records the displacement of a selected downhill branch
seed during finalization.

Let $\mathcal R_X^{\mathrm{num}}$ denote this configured initialization,
search, and native acceptance pipeline. It accepts a configuration
$X+D$ and returns the numerical saddle output when the native flags and
cap checks pass, or $\bot$ otherwise. For completed search indices $I$,
\begin{equation}
 A_I=\{\mathcal R_X^{\mathrm{num}}(X+D_m^{\mathrm{exec}}):
       m\in I,\ \mathcal R_X^{\mathrm{num}}(X+D_m^{\mathrm{exec}})
                    \ne\bot\}.
 \label{eq:rev-discoveries}
\end{equation}
The main-text ideal map specifies the saddle and connectivity target.
The numerical map implements the configured search and native acceptance
checks. Its connection to reference recovery is formalized by the local-success
assumption for this map and the final reference matcher.

\subsection{Search-force accounting and run configuration}
Each entry to \texttt{Dimer.dimer\_force} increments the search counter
once per routine invocation. The ledger covers preloading and search work,
including failed, duplicate, and capped searches with completed records.
LAMMPS initialization and native \texttt{SPSRELAX} force calls lie outside
this routine.

Reproduction uses the native helpers together with each run's resolved
step-size and decay settings, axis controls, fixed-coordinate masks,
displacement cuts, and timer boundaries. The specifications above document
the numerical path and default parameter families.

\stopcontents[appendixoutline]
\resumecontents[appendixoutlinecontinued]
\section{KDB Event-Reuse Adaptation}
\label{app:rev-kdb}
\label{app:rev-reuse}

\subsection{Stored mechanisms and the scope of the baseline}
The KDB comparison adapts the kinetic-database approach
\citep{terrell2012database}. Insertion and querying retain the upstream
geometric kernel at commit \texttt{a03dd2f}, including periodic clumping,
graph mapping, feasibility tests, and overlap checks.

The template library is built exclusively from \texttt{train\_strict}.
A stored mechanism contains its training reactant $X^{\mathrm{lib}}$,
saddle $S^{\mathrm{lib}}$, product $X^{+,\mathrm{lib}}$, mobile-atom
indices, periodic clump orders, barriers, and source/event identifiers.
The in-memory adapter retains every successfully inserted template and
records failed insertions separately. Candidate deduplication is performed
after querying as specified below.

KDB-494 denotes 494 stored event templates. Every template is queried in
both directions: one directed entry uses the training reactant as its
minimum, and the other uses the training product. Both share the stored
saddle and swap their minimum/product roles and clump order. Thus the
494-template library provides 988 directed query entries. Reverse barriers
equal the forward barrier minus the stored product energy bias. Lookup and
ranking use the target initial state and the training-library records.

\subsection{Mobile atoms and the retained local template}
Let $\operatorname{MIC}_H$ denote the periodic displacement convention
for the stored cell $H$, with row-vector coordinates as in the background.
For each training atom, define
\begin{equation}
 \begin{split}
 m_a=\max\{&
 \|\operatorname{MIC}_H(S_a^{\mathrm{lib}}-X_a^{\mathrm{lib}})\|_2,\;
 \|\operatorname{MIC}_H(S_a^{\mathrm{lib}}-X_a^{+,\mathrm{lib}})\|_2,\\
 &\|\operatorname{MIC}_H(X_a^{+,\mathrm{lib}}-X_a^{\mathrm{lib}})\|_2\}.
 \end{split}
 \label{eq:app-kdb-mobile}
\end{equation}
The mobile set consists of atoms with $m_a>0.7\,\text{\AA}$. If it is
empty, insertion retains the atom with largest reactant-to-product
displacement. The local template retains the mobile atoms and their
neighbors in any of the reactant, saddle, or product states. The neighbor
threshold uses the upstream elemental radii:
\begin{equation}
 r_{ab}^{\mathrm{neigh}}
 =(r_a^{\mathrm{elem}}+r_b^{\mathrm{elem}})(1+n_f),
 \qquad n_f=0.2.
 \label{eq:app-kdb-neighbor-cutoff}
\end{equation}
The upstream routine applies its periodic unwrapping/clumping procedure
and size checks using these elemental radii.

\subsection{Graph matching and geometric screening}
For the target initial state $X$, the upstream kernel builds the reactant
graph and searches subgraph mappings for the directed library entries.
Each surviving mapping is subjected to the upstream periodic ordered
clump construction, geometric alignment, and feasibility/overlap checks.

The acceleration layer caches target distances and applies the same
necessary pair-distance condition earlier during mapping. For a newly
mapped atom and each already mapped atom, at least one of the upstream
periodic-image distances must satisfy
\begin{equation}
 \bigl|d_{\mathrm{target}}^{(j)}-d_{\mathrm{template}}\bigr|
           <2d_c,
 \qquad d_c=0.3\,\text{\AA}.
 \label{eq:app-kdb-prefilter}
\end{equation}
The image index $j$ ranges over the upstream
\texttt{atomAtomPermDistance} candidates. Caching and JIT execution
accelerate this pair-distance check while preserving its image candidates;
surviving mappings pass the full \texttt{feasibility()} path.

\paragraph{Orthogonal alignment and the geometric gate.}
After correspondence and periodic clumping have been established, let
$A$ and $B$ be the centered template-minimum and mapped target-coordinate
matrices. Compute
\begin{equation}
 A^\top B=U\Sigma W^\top,\qquad \mathsf O=UW^\top.
 \label{eq:app-kdb-svd}
\end{equation}
The orthogonal map $\mathsf O$ acts on row vectors from the right and
permits both rotations and reflections when aligning a stored template
to the query configuration.

The alignment is computed by an SVD least-squares fit and then screened
by the maximum atomwise residual,
\begin{equation}
 d_{\mathrm{fit}}=
       \max_a\|(A\mathsf O-B)_a\|_2\le d_c.
 \label{eq:app-kdb-fit}
\end{equation}
The SVD supplies the least-squares alignment, and $d_{\mathrm{fit}}$ supplies
the maximum-residual query score and acceptance gate.


For a nearly coplanar template, the kernel also tries the normal-flipped
alignment. The coplanarity test is
\[
 \frac{\sigma_{\min}(A^\top A)}{n_{\mathrm{tpl}}}<d_c^2,
\]
where $n_{\mathrm{tpl}}$ is the retained template size. The alternative
is obtained by changing the sign of the final row of $W^\top$ before
forming $\mathsf O$. Alignments passing the distance gate subsequently
undergo the upstream product/mobile overlap checks.

\subsection{Aligned state transfer and displacement extraction}
Let $\mu_{\mathrm{tpl}}$ and $\mu_X$ be the fitted template-minimum and
mapped-target centroids, and let $\iota(a)$ be the target atom mapped to
template atom $a$. For a surviving unconstrained mapping, the aligned
saddle coordinates are placed into a copy of the target initial state:
\begin{equation}
 \widetilde S_{\iota(a)}
 =(S_a^{\mathrm{lib}}-\mu_{\mathrm{tpl}})\mathsf O+\mu_X,
 \qquad
 \widetilde S_b=X_b\quad\text{for unfilled atoms }b.
 \label{eq:app-kdb-transfer}
\end{equation}
The training product is transferred in the same way. Native constraints,
periodic reboxing, and overlap checks remain part of the upstream path.
The wrapper extracts the saddle and product displacement fields relative
to $X$. In the supplied cell representation, its extraction is
\begin{equation}
 u=\widetilde S-X,\qquad
 D^{\mathrm{cand}}=u-\operatorname{rint}(uH^{-1})H,
 \label{eq:app-kdb-displacement-extraction}
\end{equation}
and analogously for the product. This fractional-rounding extraction uses
the supplied cell representation.

Before the final reboxing, the mapped saddle displacement decomposes as
\begin{equation}
 \begin{split}
 \widetilde S_{\iota(a)}-X_{\iota(a)}
 ={}&(S_a^{\mathrm{lib}}-X_a^{\mathrm{lib}})\mathsf O\\
 &+\bigl[(X_a^{\mathrm{lib}}-\mu_{\mathrm{tpl}})\mathsf O
                      +\mu_X-X_{\iota(a)}\bigr].
 \end{split}
 \label{eq:app-kdb-fit-residual}
\end{equation}
The bracket is the initial-state fitting residual. KDB coordinate transfer
therefore combines the transformed stored displacement with this alignment
correction. KeyBank uses the direct displacement transfer described below.

\subsection{Global candidate deduplication and Dimer submission}
Entry-internal deduplication is disabled in the query path.
After all directed entries have been queried, candidates are sorted by
increasing geometric fit score and greedily deduplicated. Let
$D^{\mathrm{cand}}$ and $D^{+,\mathrm{cand}}$ be one candidate's saddle
and product displacements. It is discarded when an already retained
candidate $j$ satisfies both
\begin{equation}
 \max_a\|D^{\mathrm{cand}}_a-D^{\mathrm{cand}}_{j,a}\|_2
                    \le0.3\,\text{\AA},\qquad
 \max_a\|D^{+,\mathrm{cand}}_a-D^{+,\mathrm{cand}}_{j,a}\|_2
                    \le0.3\,\text{\AA}.
 \label{eq:app-reuse-deduplication}
\end{equation}
These fields use the target atom correspondence. Score-ordered greedy
deduplication compares each candidate against retained candidates, using
the transferred saddle and product predictions from the training library.
Source, event, and direction identifiers are preserved.

Retained saddle displacements initialize native Dimer searches with a
300-translation cap. Physical acceptance and final recovery use the common
protocol, including the all-atom $0.05\,\text{\AA}$ displacement and
$0.02\,\mathrm{eV}$ barrier gates. Library size, directed entries, returned
proposals, and recovered references summarize successive pipeline stages.

\section{Rooted-Topology KeyBank Adaptation}
\label{app:rev-keybank}

\subsection{Rooted graphs and directed library entries}
KeyBank adapts the rooted-topology catalogue idea of k-ART
\citep{beland2011kinetic} to retrieve proposals for the common native Dimer
pipeline. KeyBank-32 and KeyBank-128 use libraries built from 32 and 128
\texttt{train\_strict} configurations, respectively.

For a root atom $r$ in a configuration, retain
\begin{equation}
 \mathcal N_r=\{a:\|X_a-X_r\|_2\le5\,\text{\AA}\},\qquad
 p_a=X_a-X_r\quad(a\in\mathcal N_r).
 \label{eq:app-keybank-root-neighborhood}
\end{equation}
The local graph has an atom-species/root label
$(z_a,\mathbf1\{a=r\})$ on each vertex and an undirected edge when
\begin{equation}
 \|p_a-p_b\|_2<3.145\,\text{\AA},\qquad a\ne b.
 \label{eq:app-keybank-bond}
\end{equation}
The neighborhood radius is inclusive and the bond cutoff is strict.
A four-iteration Weisfeiler--Lehman hash with node labels identifies lookup
buckets; exact graph mappings and geometric checks then resolve candidates.

KeyBank operates directly on local active-volume coordinates, with the
stated assumption that the represented neighborhoods remain within a
consistent cell image.

\paragraph{Forward and reverse library entries.}
For every training event, construct a forward entry with initial state
$X^{\mathrm{lib}}$ and a reverse entry with initial state
$X^{+,\mathrm{lib}}$. In either direction, let $X^{\mathrm{dir}}$ be the
chosen initial state, $X^{+,\mathrm{dir}}$ the opposite minimum, and
$S^{\mathrm{lib}}$ the shared saddle. The full directed fields are
\[
 D^{\mathrm{lib}}=S^{\mathrm{lib}}-X^{\mathrm{dir}},\qquad
 D^{+,\mathrm{lib}}=X^{+,\mathrm{dir}}-X^{\mathrm{dir}}.
\]
Choose the training root as the atom with largest
$\|D_a^{\mathrm{lib}}\|_2$, using the stored-order first maximum at a tie.
The rooted graph is constructed from the directed initial state $X^{\mathrm{dir}}$.
An entry is retained only when the saddle-displacement tail outside its
$5\,\text{\AA}$ neighborhood satisfies
\begin{equation}
 \max_{a\notin\mathcal N_r}\|D_a^{\mathrm{lib}}\|_2
                         <0.3\,\text{\AA},
 \label{eq:app-keybank-tail}
\end{equation}
with an empty tail assigned value zero. Entries with equality at the
threshold are rejected. The stored record contains local initial
coordinates, the colored graph, saddle and product displacements within
the neighborhood, source/event identifiers, direction, and tail maximum.
The entry filter applies the stated saddle-displacement tail criterion.

\subsection{Exact graph mappings and necessary distance checks}
At query time, \emph{every} target atom is considered as a root. Target
root selection uses no reference moving atoms, saddle, or product.
For entries in the corresponding WL bucket, exact colored graph
isomorphism is tested using the VF2 matcher. The root label
must match the root label and chemical species must match.

The geometric matcher additionally prunes a candidate correspondence
$(i,j)$ against every already mapped pair $(a,b)$ unless
\begin{equation}
 \bigl|\|p_i-p_a\|_2-\|q_j-q_b\|_2\bigr|
                              \le0.600001\,\text{\AA}.
 \label{eq:app-keybank-pair-pruning}
\end{equation}
The $0.6\,\text{\AA}$ scale is twice the $0.3\,\text{\AA}$ atomwise
fit tolerance, with the implementation's small numerical allowance.
This pair-distance filter precedes the final fit gate. Candidate vertices
are ordered by root status, root-relative distance, and index, followed by
exact colored-isomorphism and geometric tests.

\subsection{Orthogonal fit and direct displacement transfer}
For an exact graph mapping, reorder the target local coordinates to the
template order. Subtract each local point set's own centroid, obtaining
$A$ and $B$, and form
\[
 A^\top B=U\Sigma W^\top,\qquad \mathsf O=UW^\top.
\]
The orthogonal fit permits both rotations and reflections. Its acceptance
criterion is
\begin{equation}
 \max_a\|(A\mathsf O-B)_a\|_2\le0.3\,\text{\AA}.
 \label{eq:app-keybank-fit}
\end{equation}
The maximum residual on the least-squares fit serves as both acceptance
criterion and candidate score.

KeyBank transfers stored displacement vectors by the fitted orthogonal
map. With target indices $\iota(a)$,
\begin{equation}
 \begin{aligned}
 D^{\mathrm{cand}}_{\iota(a)}&=D_a^{\mathrm{lib}}\mathsf O,&
 D^{+,\mathrm{cand}}_{\iota(a)}&=D_a^{+,\mathrm{lib}}\mathsf O,\\
 D^{\mathrm{cand}}_b&=0,&D^{+,\mathrm{cand}}_b&=0
       \quad\text{for all unfilled target atoms }b.
 \end{aligned}
 \label{eq:app-keybank-transfer}
\end{equation}
These fields directly transfer the library motion to matched target atoms,
with zero displacement assigned to the remaining atoms.

\subsection{Mapping limit, sorting, and candidate deduplication}
The query has a global mapping-enumeration protection limit of
100,000 per target configuration, accumulated across roots and library
entries. The counter increments for each yielded graph mapping before
the alignment gate. When it exceeds the limit, enumeration stops and
\texttt{truncated=True} records that retrieval ended at the mapping limit.

Candidates are sorted lexicographically by geometric fit score, source,
event, reverse-direction flag, and query root. They are then greedily
deduplicated against already retained candidates using both the saddle
and product conditions in
Equation~\revEqref{eq:app-reuse-deduplication}, with
$0.3\,\text{\AA}$ for each. The transferred library product provides the
product field used in these candidate-to-candidate comparisons.

Retained fields initialize native Dimer searches capped at 300 translations.
Native acceptance and common reference matching determine recovery using
the all-atom $0.05\,\text{\AA}$ displacement and $0.02\,\mathrm{eV}$ barrier
criteria. Retrieval uses the graph, fit, tail, and deduplication settings above.

\section{Shared Multiscale Embedding of Search Trajectories}
\label{app:rev-trajectory-mds}

\paragraph{Complementary views of the energy landscape and search paths.}
Figure~\ref{fig:rev-recovery-cost}(a) presents the Sammon-projected energy
background with initial, saddle, and product configurations. Panel~(b)
compares ASPIRE and conventional Dimer trajectories in one shared MDS
coordinate system. The multiscale construction for panel~(b) is detailed below.

\subsection{Illustrative sample and one shared coordinate system}
Figure~\ref{fig:rev-recovery-cost}(b) embeds 33 saved successful searches:
17 ASPIRE and 16 Random trajectories, totaling 4,994 frames. This selected
sample visualizes search geometry; the quantitative recovery and cost
experiments use their separately specified evaluation populations.

Both methods, all retained frames, and the geometric landmarks share one
embedding with fixed atom identities and ordering. For a trajectory with
stored initial configuration $X^{(\ell)}$, a single translation
\[
 t_\ell=\frac1N\sum_a(X_a-X_a^{(\ell)})
\]
aligns that initial configuration to the shared native state $X$.
Agreement is checked after the shift, which is then applied unchanged
to every full-atom frame of that trajectory.

\subsection{Configuration features, landmarks, and radial preprocessing}
For an aligned configuration $Z$, define
\begin{equation}
 \mathbf f(Z)=\frac{\operatorname{vec}(Z-X)}{\sqrt N}.
 \label{eq:app-mds-cartesian-features}
\end{equation}
Consequently,
\begin{equation}
 \|\mathbf f(Z_i)-\mathbf f(Z_j)\|_2
 =\left(\frac1N\sum_a\|Z_{i,a}-Z_{j,a}\|_2^2\right)^{1/2}.
 \label{eq:app-mds-rms-metric}
\end{equation}
This metric is the Cartesian RMS over all corresponding atoms.
Coordinates and energies retain their recorded atom correspondence.

The reference list contains the native initial state, followed by
41 separately supplied landmarks: another initial-state record,
20 saddles, and 20 products. In the code's zero-based reference list,
\texttt{refs[2:22]} selects the 20 saddle landmarks. Exactly equal feature
rows across landmarks and trajectories share one two-dimensional variable,
with first-occurrence ordering preserved. Associated energy records are
retained separately for coloring.

\paragraph{Radial compression before dimensionality reduction.}
Let the reference radius be the median saddle distance from the initial
configuration in the full feature space,
\begin{equation}
 r_{\mathrm{ref}}
 =\operatorname{median}_{e=1}^{20}
       \|\mathbf f(S_e^{\mathrm{landmark}})\|_2,
 \qquad r_{\mathrm{map}}=0.4r_{\mathrm{ref}}.
 \label{eq:app-mds-reference-radius}
\end{equation}
The displayed selected map first applies the radial transformation
\begin{equation}
 \mathbf G(\mathbf f)=
 \begin{cases}
 \displaystyle r_{\mathrm{map}}
  \operatorname{asinh}\!\left(\frac{\|\mathbf f\|_2}{r_{\mathrm{map}}}\right)
  \frac{\mathbf f}{\|\mathbf f\|_2},&\|\mathbf f\|_2>0,\\
 0,&\mathbf f=0.
 \end{cases}
 \label{eq:app-mds-radial-map}
\end{equation}
The common radial transformation acts in the full $3N$-dimensional feature
space before the shared two-dimensional fit. Fitting targets are distances
between the transformed features $\mathbf G_i$. The display thus provides
a compressed geometric view of the trajectories, with axes expressed in
arbitrary units and physical search performance reported by the recovery
and force-call metrics.

\subsection{Six classes of distance constraints}
Let the deduplicated feature nodes be indexed by $i$. The edge collection
is constructed once from the original features $\mathbf f_i$ with random
seed 20260924, and is kept fixed when changing the fitted metric. It
contains the following six classes.

\paragraph{Global pairs.}
Generate 60,000 pairs by sampling a trajectory uniformly and then a
saved frame uniformly within that trajectory, independently for the
two endpoints. Self-pairs are removed. Uniform trajectory selection balances
the contribution of the 33 selected successful searches at this stage.

\paragraph{Frame--reference pairs.}
Connect every saved frame of each trajectory to all reference-list
entries. Each trajectory forms a separate normalized block with initial
importance $1/33$.

\paragraph{Adjacent-step pairs.}
For each trajectory, connect all consecutive saved frames. The input
check requires consecutive recorded translation-step indices. Each
trajectory's adjacent-step block has initial importance $1/33$.

\paragraph{Temporal-window pairs.}
For each trajectory and each lag in $\{2,4,8,16,32\}$ shorter than the
trajectory, connect frames at that lag. Each available block has initial
importance $1/(33\times5)$.

\paragraph{Reference--reference pairs.}
Include all unordered pairs of reference-list entries, removing pairs
whose node identifiers are identical. The block has initial importance
one.

\paragraph{Original-space nearest neighbors.}
Connect each deduplicated feature node to five nearest neighbors in the
original Euclidean feature space. This couples nearby configurations
across trajectories and methods, with initial block importance one.
Neighbor selection uses the original Euclidean feature distances.

Each block removes self-pairs and retains its own weights. Contributions
from multiple blocks together define the weighted geometric constraints.

\subsection{Block normalization and weighted stress}
The optimizer uses the internal numerical scale
$s_{\mathrm{num}}=0.05$. Define original and compressed target distances
in those internal coordinates by
\begin{equation}
 d^0_{ij}=\frac{\|\mathbf f_i-\mathbf f_j\|_2}{s_{\mathrm{num}}},\qquad
 d_{ij}=\frac{\|\mathbf G_i-\mathbf G_j\|_2}{s_{\mathrm{num}}}.
 \label{eq:app-mds-internal-distances}
\end{equation}
For a block $B$, the base weighting is
\begin{equation}
 b_B(d)=
 \begin{cases}
 1/\max(d,2\times10^{-4}),
       &B\text{ is frame--reference or nearest-neighbor},\\
 1,&\text{otherwise}.
 \end{cases}
 \label{eq:app-mds-local-weight}
\end{equation}
The $2\times10^{-4}$ stabilizing floor is expressed in the optimizer's
internally scaled distance units.

Let $\alpha_B$ be the initial block importance above, with
$\alpha_B=1$ for the global block. Original-space weights are
\begin{equation}
 w^0_{ij,B}
 =\frac{\alpha_B b_B(d^0_{ij})}
 {\max\!\left(\sum_{(u,v)\in B}b_B(d^0_{uv})(d^0_{uv})^2,10^{-15}\right)}.
 \label{eq:app-mds-original-weights}
\end{equation}
When substituting the compressed metric, the implementation preserves
$\alpha_B^{\mathrm{eff}}=\sum_{(i,j)\in B}w^0_{ij,B}(d^0_{ij})^2$,
which equals the intended $\alpha_B$ when the normalization floor is
inactive, and recomputes
\begin{equation}
 w_{ij,B}
 =c_B\frac{\alpha_B^{\mathrm{eff}}b_B(d_{ij})}
 {\max\!\left(\sum_{(u,v)\in B}b_B(d_{uv})d_{uv}^2,10^{-15}\right)},
 \qquad
 c_B=\begin{cases}
 12,&\text{adjacent step},\\
 6,&\text{temporal window},\\
 1,&\text{all other blocks}.
 \end{cases}
 \label{eq:app-mds-selected-weights}
\end{equation}
These block weights balance trajectory, temporal, local, and global
geometric constraints in the shared fit.

The two-dimensional variables $Y_i\in\mathbb R^2$ minimize the weighted
metric stress
\begin{equation}
 \mathcal S(Y)=\sum_B\sum_{(i,j)\in B}
 w_{ij,B}\bigl(\|Y_i-Y_j\|_2-d_{ij}\bigr)^2.
 \label{eq:app-mds-stress}
\end{equation}
The implementation accumulates the two opposite endpoint contributions
of each edge using the stabilized expression
\[
 2w_{ij,B}\frac{\widetilde r_{ij}-d_{ij}}{\widetilde r_{ij}}
          (Y_i-Y_j),\qquad
 \widetilde r_{ij}=\max(\|Y_i-Y_j\|_2,10^{-12}).
\]
Numerical stress evaluation also uses $\widetilde r_{ij}$. Above the
$10^{-12}$ floor, the expression is the stress gradient; the floor
stabilizes updates for coincident or nearly coincident nodes. Fitting uses
geometric constraints from the selected successful trajectories and landmark
coordinates. Energies provide the separate color overlay described below.

\subsection{Initialization, optimization, and reproducibility}
The selected map starts from a saved shared raw-metric MDS embedding.
That baseline fits the uncompressed distances using the same edge
construction, with three stages of adjacent-step multipliers $2,8,24$
and temporal-window multipliers equal to half those values. Each baseline
stage allows up to 900 L-BFGS-B iterations. Its saved initial coordinates
come from the common Sammon map in this workflow.

For the displayed map, the saved two-dimensional baseline is radially
compressed with $r_{\mathrm{map}}$ and divided by $s_{\mathrm{num}}$
for initialization. The compressed-space distances and weights in
Equations~\revEqref{eq:app-mds-internal-distances}--\revEqref{eq:app-mds-selected-weights}
are then optimized in two consecutive L-BFGS-B stages, each allowing
800 iterations. The optimization settings are
\texttt{ftol}$=2\times10^{-10}$, \texttt{gtol}$=10^{-6}$,
\texttt{maxcor}$=12$, and \texttt{maxls}$=30$.
The initial-state node is recentered to the origin after each stage.

Both selected-map stages ran to their 800-iteration limits. Reproduction
uses the saved baseline and initial coordinates together with the frames,
landmarks, sampled edges, and hyperparameters. Final stress and distance
distortion are evaluated from the saved optimization outputs.

\subsection{Common display alignment, markers, and energy colors}
After optimization, a single common orthogonal Procrustes alignment of
the landmark coordinates places the map in the orientation of the saved
baseline. No fitted scale factor is applied. The display coordinates are
\begin{equation}
 y_i^{\mathrm{display}}
 =\frac{s_{\mathrm{num}}}{r_{\mathrm{ref}}}
       (Y_i-Y_{\mathrm{initial}})\mathsf O_{2D},
 \label{eq:app-mds-display}
\end{equation}
so the axes are labeled in arbitrary units (a.u.). Both methods share this
transformation and scale; panel view limits can provide a closer view of
the ASPIRE region alongside the wider Random trajectories.

Lines retain all saved frames. Markers highlight translation-step indices
divisible by five, together with each trajectory's first and last frames.

\section{Theoretical Analysis}
\label{app:rev-theory}

\subsection{Local convergence motivation}
\label{app:rev-local-convergence}
\begin{proposition}[Local convergence,
    {\citealp[Theorem~4.2]{luo2022convergence}}]
\label{prop:rev-ideal-radius}
In fixed, vectorized free coordinates, let $s$ be a nondegenerate
first-order saddle of $V$. Suppose $V\in C^2$ and, throughout
$B(s,R_0)$, its Hessian is $L_3$-Lipschitz, has exactly one negative
eigenvalue, and satisfies $\mu\leq|\lambda_i|\leq L$ for all $i$,
where $0<\mu<L$. We take $L_3>0$. Consider the ideal minimum-mode
iteration
\begin{equation}
\label{eq:rev-idealiteration}
y_{j+1}
= y_j-\frac{2}{L+\mu}
  \bigl(I-2v_1(y_j)v_1(y_j)^\top\bigr)\nabla V(y_j),
\end{equation}
where $v_1(y_j)$ is an exact unit eigenvector associated with
the lowest Hessian eigenvalue. Writing $r_j=\|y_j-s\|_2$, if
$r_0<\min\{R_0,\rho\}$, $\rho=2\mu/L_3$, then the iterates remain
in the neighborhood and converge geometrically to $s$, with
$r_j\leq\left(\frac{L+\mu}{L+3\mu}\right)^j
\frac{\rho r_0}{\rho-r_0}$.
\end{proposition}

\subsection{Proof of Geometric Consistency}
\label{app:rev-equivariance}
The positions and vector-channel rows use the main-text row-vector convention. The input transformation is
\begin{equation}
 X'=PXR^\top,\qquad R\in\mathrm{SO}(3),
 \label{eq:app-input-action}
\end{equation}
where $P$ reindexes atoms and all atom-indexed auxiliary features consistently. We establish rotation and atom-permutation equivariance for deterministic inference on the prepared Cartesian representation.

\begin{lemma}[Attention and weighted readouts]
\label{lem:rev-attention}
Let $A=\operatorname{softmax}_{\mathrm{row}}(L)$. If $L'=LP^\top$, scalar atomic features satisfy $Z'=PZ$, and each vector channel satisfies $V'=PVR^\top$, then
\begin{equation}
 A'=AP^\top,\qquad A'Z'=AZ,\qquad A'V'=AVR^\top.
 \label{eq:app-attention-law}
\end{equation}
The vector identity also applies to $X'=PXR^\top$.
\end{lemma}
\begin{proof}
Row-wise softmax commutes with a permutation of its entries. Thus $A'=AP^\top$. The two readout identities follow by multiplication and $P^\top P=I$. A vector-channel projection acts on channels, so it commutes with right multiplication by $R^\top$.
\end{proof}

\begin{lemma}[Slot transformation law]
\label{lem:rev-slot-geometry}
Suppose encoder scalars are invariant under rotations, encoder vector channels rotate covariantly, and both are reindexed by $P$. The slot stages and final readout satisfy, for each fixed slot $m$,
\begin{equation}
 s'_m=s_m,\qquad z'_m=z_m,\qquad
 Q'_m=Q_mR^\top,\qquad A'=AP^\top.
 \label{eq:rev-slot-law}
\end{equation}
\end{lemma}
\begin{proof}
The learned scalar seeds and their scalar self-attention are unchanged by input rotations and atom permutations. The zero initial vector state is covariant. At initialization, scalar query--key logits only reindex across atoms, so Lemma~\ref{lem:rev-attention} gives the required vector law for $Q_m^{(1)}$.

Assume the transformation laws hold at a vector--scalar stage. The scalar part of Equation~\revEqref{eq:rev-slot-attn} is invariant, and the vector part is invariant because
\[
 \langle UR^\top,VR^\top\rangle_F=\langle U,V\rangle_F.
\]
Consequently, the attention weights reindex with the atoms. The pooled scalar context is invariant, and the vector context rotates by $R^\top$. Vector norms are invariant, so the MLP update preserves the scalar law; bias-free channelwise vector maps preserve the vector law. This proves the induction step for both vector--scalar stages.

The subsequent scalar-only stages use invariant fingerprints and atomic scalars. Their headwise dot products, normalizations, attention-weighted sums, and feed-forward updates therefore preserve scalar invariance and leave the vector state unchanged. The final readout uses invariant scalar products and vector inner products again. Applying Lemma~\ref{lem:rev-attention} to its contexts yields Equation~\revEqref{eq:rev-slot-law}.
\end{proof}

\begin{proposition}[Proposal equivariance]
\label{prop:rev-equivariance}
Assume deterministic inference, rotation- and permutation-compatible graph construction and neighbor selection, an equivariant encoder, the stated slot updates, invariant scalar decoder conditioning, bias-free channelwise vector maps, and the radial normalization in Equation~\revEqref{eq:rev-bases}. Then every fixed slot satisfies
\begin{equation}
 \widehat D_m(PXR^\top)=P\widehat D_m(X)R^\top,\qquad
 \widehat S_m(PXR^\top)=P\widehat S_m(X)R^\top.
 \label{eq:rev-equivariance}
\end{equation}
\end{proposition}
\begin{proof}
Corresponding Cartesian edges rotate as $r'_{ba}=r_{ba}R^\top$ after atom reindexing. Their lengths and radial features are unchanged, and their spherical harmonics transform in the appropriate rotation representations. Compatible neighbor selection preserves the same reindexed graph. The equivariant encoder therefore supplies the scalar and vector laws required by Lemma~\ref{lem:rev-slot-geometry}.

The final attention satisfies $A'=AP^\top$. Hence
\[
 b'_m=\sum_aA'_{ma}X'_a=b_mR^\top.
\]
The offsets $X_a-b_m$ rotate by $R^\top$ under the same atom reindexing. The 32 atom-derived vectors and eight slot-derived vectors also rotate, since their learned maps only mix vector channels. The stabilized direction satisfies
\[
 \operatorname{unit}_\epsilon(vR^\top)
 =\frac{vR^\top}{\max(\|v\|_2,\epsilon)}
 =\operatorname{unit}_\epsilon(v)R^\top,
\]
including at $v=0$. Thus every row of $\mathcal B_{ma}$ is covariant.

All scalar entries of $\xi_{ma}$ are invariant: they are invariant scalar representations, inner products, or squared norms. The shared coefficient networks therefore produce the same $p_{ma},\kappa_{ma},q_{ma},q^{\mathrm{res}}_{ma}$ at corresponding atom--slot pairs. Both the gated directional term and the ungated residual in Equation~\revEqref{eq:rev-decoder} rotate covariantly. This proves the displacement law. Adding $X'=PXR^\top$ gives the configuration law.
\end{proof}

\paragraph{Representation and implementation conditions.}
The deterministic inference setting uses atom-reindexing-compatible graph truncation, disabled dropout, and a common Cartesian chart. The arithmetic anchor is evaluated in that chart as described in Appendix~\ref{app:rev-periodic-background}, with common origin changes handled during input preparation. These conditions establish the neural proposal transformation law.

\section{Single-Step Reference Interpretation of RWC}
\label{app:rev-rwc}

\subsection{Reference rates and restricted one-step models}
Fix a state $X$ and a finite, nonempty reference set $\mathcal E_X^{\mathrm{ref}}$ with positive rates $k_e$. For a nonempty subset $G\subseteq\mathcal E_X^{\mathrm{ref}}$, write
\begin{equation}
 \Lambda_G=\sum_{e\in G}k_e,\qquad
 p_G(e)=\frac{k_e}{\Lambda_G}\quad(e\in G),\qquad
 \Delta t_G\sim\operatorname{Exp}(\Lambda_G).
 \label{eq:rev-set-clock}
\end{equation}
These are the single-step event distribution and waiting time described in Appendix~\ref{app:rev-kmc} \citep{voter2007introduction}, providing a kinetic interpretation of reference recovery at a fixed state.

\begin{proposition}[Single-step interpretation of RWC]
\label{prop:rev-rwc}
Let the recovered reference subset $\mathcal E_X^{\mathrm{rec}}$ be nonempty and retain the same rates as the full reference set. Extend its event distribution by zero outside that subset. Then
\begin{equation}
 \operatorname{TV}(p_{\mathrm{ref}},p_{\mathrm{rec}})
 =1-\mathrm{RWC}_X,\qquad
 \frac{\mathbb E[\Delta t_{\mathrm{rec}}]}
      {\mathbb E[\Delta t_{\mathrm{ref}}]}
 =\frac1{\mathrm{RWC}_X},
 \label{eq:rev-rwc-interpretation}
\end{equation}
where $p_{\mathrm{ref}}=p_{\mathcal E_X^{\mathrm{ref}}}$ and
$p_{\mathrm{rec}}=p_{\mathcal E_X^{\mathrm{rec}}}$.
\end{proposition}
\begin{proof}
Put $c=\Lambda_{\mathrm{rec}}/\Lambda_{\mathrm{ref}}=\mathrm{RWC}_X\in(0,1]$. On the recovered subset, $p_{\mathrm{rec}}(e)=p_{\mathrm{ref}}(e)/c\ge p_{\mathrm{ref}}(e)$. Consequently,
\begin{align*}
 \operatorname{TV}(p_{\mathrm{ref}},p_{\mathrm{rec}})
 &=\frac12\left[
 \sum_{e\in\mathcal E_X^{\mathrm{rec}}}
     (p_{\mathrm{rec}}(e)-p_{\mathrm{ref}}(e))
 +\sum_{e\notin\mathcal E_X^{\mathrm{rec}}}p_{\mathrm{ref}}(e)
 \right]\\
 &=\frac12[(1-c)+(1-c)]=1-c.
\end{align*}
The exponential mean is $1/\Lambda_G$, so the ratio of waiting-time means is $\Lambda_{\mathrm{ref}}/\Lambda_{\mathrm{rec}}=1/c$.
\end{proof}

\subsection{Average recovery probability and the high-temperature limit}
Sample one reference event $E_X$ at each fixed evaluation state using its normalized reference rates. Then
\[
 \Pr(E_X\in\mathcal E_X^{\mathrm{rec}})=\mathrm{RWC}_X.
\]
For the fixed evaluation bank $\mathcal X$, linearity of expectation gives
\begin{equation}
 \mathbb E\!\left[\frac1{|\mathcal X|}\sum_{X\in\mathcal X}
 \mathbf1\{E_X\in\mathcal E_X^{\mathrm{rec}}\}\right]
 =\frac1{|\mathcal X|}\sum_{X\in\mathcal X}\mathrm{RWC}_X
 =\overline{\mathrm{RWC}}.
 \label{eq:app-rwc-expectation}
\end{equation}
Linearity of expectation makes this identity valid for dependent samples as well. Evaluation configurations receive uniform weight, consistently with the main-text metric.

Under the main-text common prefactor, $k_e=\nu\exp[-\Delta E_e/(k_{\mathrm B}T)]$. For a fixed finite reference set and a fixed recovered subset, every exponential approaches one as $T\to\infty$. Therefore
\[
 \lim_{T\to\infty}\mathrm{RWC}_X
 =\frac{|\mathcal E_X^{\mathrm{rec}}|}{|\mathcal E_X^{\mathrm{ref}}|}
 =\mathrm{Cov}_X.
\]
This formal limit describes how the weights of a fixed catalogue become uniform under the common-prefactor convention. With event-specific prefactors, the limiting normalized weights are proportional to those prefactors.

\subsection{Metric interpretation and conventions}
The proposition uses fixed reference rates. Numerical barriers enter the matching gate, while matched events retain their reference rates in the sums. For an empty recovered set, RWC is zero and the restricted exit distribution and finite mean waiting time are undefined. Aggregating event identities into product-basin labels can decrease total variation, yielding a corresponding comparison at the basin level.

RWC measures the fraction of reference escape-rate mass recovered at each state. A common scaling of all rates preserves this fraction and rescales the clock, separating relative event recovery from absolute waiting times. Averaging RWC over the fixed state bank summarizes this reference-relative recovery across configurations.

\section{Physical Evaluation, Event Matching, and Cost Accounting}
\label{app:rev-matching}

\subsection{Fixed 50-configuration evaluation set and topology coverage}
\label{app:rev-eval50}
The ablation and wall-time experiments use the same fixed set
$\mathcal X_{50}$ of 50 held-out configurations. It augments an existing
30-configuration evaluation set with 20 additional cases selected to broaden
topology coverage. The set is drawn from the held-out configuration bank;
the judge uses a separately defined 50-configuration monitoring subset.

The five vacancy-connectivity classes are the component-size signatures
recorded as \texttt{fragmentation\_signature}:
$4$, $3+1$, $2+2$, $2+1+1$, and $1+1+1+1$. For example, $3+1$ denotes
one connected component of three vacancies and one isolated vacancy,
whereas $1+1+1+1$ denotes four separate vacancy components. These
signatures summarize connectivity in the dataset's vacancy graph.
Each class can contain multiple geometric arrangements and fine-grained
topologies. Encoder, judge, and KeyBank atom-neighbor graphs are defined
in their respective method sections.

\begin{table}[t]
 \centering\small
 \caption{\textbf{Vacancy-connectivity coverage of the fixed evaluation subset.}
 The existing 30 cases are augmented by 20 cases to obtain the 50-case
 set used for both ablation and wall-time comparisons. Counts report the
 configuration selection in each connectivity class.}
 \label{tab:app-eval50-topology}
 \begin{tabular}{@{}lrrr@{}}
 \toprule
 Component-size signature & Existing 30 & Added 20 & Final 50 \\
 \midrule
 $4$             & 0  & 2 & 2  \\
 $3+1$           & 4  & 4 & 8  \\
 $2+2$           & 0  & 5 & 5  \\
 $2+1+1$         & 11 & 5 & 16 \\
 $1+1+1+1$       & 15 & 4 & 19 \\
 \midrule
 Total           & 30 & 20 & 50 \\
 \bottomrule
 \end{tabular}
\end{table}

After excluding the existing 30 configurations from the held-out bank,
the augmentation included all remaining cases in classes $4$ and $2+2$.
It prioritized the available rare target-PDR label combinations
$(1,3)$, $(0,3)$, and $(0,1,2,3)$, then filled the remaining positions
deterministically by configuration ID within the relevant connectivity
classes until 20 additional cases had been selected. The target-PDR
combinations refer to the stored \texttt{target\_pdr\_group} metadata;
they complement the vacancy component-size signatures in offline subset
construction. At query time, the predictor, judge, and event-reuse methods
use their specified initial-state inputs.

The 20-case augmentation uses topology, target-label coverage, and
configuration ID as its selection criteria. Together with the existing
30 cases, it forms a fixed evaluation set covering all five vacancy-connectivity
classes. Every compared method uses this same designed subset.


\subsection{Numerical acceptance and reference-event matching}
The evaluation state is the relaxed $11\,\text{\AA}$ SEAKMC active volume
$X$ in the offline protocol of Section~\ref{sec:experiments}. All methods
use the same fixed local reference catalogue $\mathcal E_X^{\mathrm{ref}}$,
so coverage measures event recovery for that active-volume configuration.

For a completed search, an output is eligible for reference matching
when the native valid and connected flags are true, the maximum-translation
flag is false, and the recorded translation count is strictly below its
configured cap. These operational acceptance criteria follow
Appendix~\ref{app:rev-refinement}; the reported pipeline uses them without
an additional independent Hessian-index check.

\paragraph{Saddle geometry, recorded barriers, and reference identities.}
The output saddle coordinates are the saved
\texttt{saddle\_active\_coords}, obtained from native \texttt{XCR.T}.
Set $D_{\mathrm{out}}=S_{\mathrm{out}}-X$ and use the associated recorded
barrier $\Delta E_{\mathrm{out}}$. Native \texttt{XDISP} records the
branch-dependent finalization displacement. For reference $e$ from the
same $X$, the eligible set is
\begin{equation}
 \begin{split}
 \mathcal M_X(S_{\mathrm{out}})=\bigl\{e:\;&
 \max_{1\le a\le N}\|D_{\mathrm{out},a}-D_{e,a}\|_2
                                      \le0.05\,\text{\AA},\\
 &|\Delta E_{\mathrm{out}}-\Delta E_e|
                                      \le0.02\,\mathrm{eV}\bigr\}.
 \end{split}
 \label{eq:app-eligible-references}
\end{equation}
This main-text criterion uses both non-strict inequalities and compares
all corresponding atoms in the fixed Cartesian representation. Matching
uses the recorded saddle displacement and barrier directly.

The matcher $J_X(S_{\mathrm{out}})$ assigns at most one reference identity
in $\mathcal M_X(S_{\mathrm{out}})$, or $\bot$ if the set is empty.
Distinct recovery is the set of assigned reference identities: repeated
hits to one reference count once. The reference catalogue remains fixed
and accepted unmatched outputs receive no reference-coverage credit.
Accepted unmatched outputs are candidates for further mechanism analysis
and expansion of the reference catalogue.



\subsection{Distinct thresholds serve distinct operations}
Training activity, judge labels, retrieval, numerical connectivity, and
reference matching are separate definitions. Their use is summarized
below to make clear which operations depend on reference data.
\begin{table}[t]
 \centering\small
 \caption{\textbf{Thresholds in the implementation.}
 ``Reference'' denotes the target configuration's reference catalogue.
 Event-reuse methods retrieve stored mechanisms from their training libraries.}
 \label{tab:app-threshold-roles}
 \begin{tabular}{@{}p{0.24\linewidth}p{0.47\linewidth}p{0.21\linewidth}@{}}
 \toprule
 Operation & Criterion / threshold & Target-reference use \\
 \midrule
 Active-aware training & Reference displacement $\ge0.05\,\text{\AA}$;
 largest-reference-atom fallback & Training only \\
 Judge movers & Proposal/reference union uses $\ge0.3\,\text{\AA}$ &
 Offline labels only \\
 Judge positive label & Nearest-union RMS $<0.35\,\text{\AA}$ and
 union maximum atom error $<0.5\,\text{\AA}$ & Offline labels only \\
 Judge selection & $\widehat p_m\ge\tau$ after full-pool features & None \\
 KDB / KeyBank retrieval & Maximum aligned atom error
 $\le0.3\,\text{\AA}$ & None \\
 Reuse candidate deduplication & Both saddle and product maximum field
 differences $\le0.3\,\text{\AA}$ & None \\
 Native connectivity & Maximum coordinate component $<0.1\,\text{\AA}$
 for the return branch, with the distinct-branch test & Initial $X$ only \\
 Final reference recovery & All-atom maximum Euclidean error
 $\le0.05\,\text{\AA}$ and barrier error $\le0.02\,\mathrm{eV}$ &
 Evaluation only \\
 \bottomrule
 \end{tabular}
\end{table}
KDB mobile cutoffs and KeyBank graph/tail criteria are specified in their
respective sections. Dataset deduplication applies the 48 bcc operations
to initial configurations, while event recovery uses the fixed-$X$
matching predicate above.

\subsection{Reference metrics and experiment-specific settings}
Coverage and RWC use the recovered-reference subset and the reference
rates, exactly as in Section~\ref{sec:experiments}:
\begin{equation}
 \begin{gathered}
 \mathrm{Cov}_X=
 \frac{|\mathcal E_X^{\mathrm{rec}}|}{|\mathcal E_X^{\mathrm{ref}}|},
 \qquad
 \mathrm{RWC}_X=
 \frac{\sum_{e\in\mathcal E_X^{\mathrm{rec}}}k_e}
      {\sum_{e\in\mathcal E_X^{\mathrm{ref}}}k_e},\\
 \overline{\mathrm{RWC}}=
 \frac1{|\mathcal X|}\sum_{X\in\mathcal X}\mathrm{RWC}_X,
 \qquad
 k_e=\nu\exp\!\left(-\frac{\Delta E_e}{k_{\mathrm B}T}\right).
 \end{gathered}
 \label{eq:rev-cov-rwc}
\end{equation}
All reported RWC values use $T=800\,\mathrm K$ and the common prefactor
$\nu$, as stated in the main text. The average weights evaluation
configurations uniformly. Matched outputs retain their assigned reference
rates, and the common prefactor cancels from RWC
(Appendix~\ref{app:rev-kmc-background}). The reference catalogue,
temperature, and prefactor convention are fixed across judge-threshold
and random-start comparisons.

\begin{table}[t]
 \centering\small
 \caption{\textbf{Settings of the main-text evaluation populations.}
 The 50-state ablation and wall-time studies use the same held-out subset,
 but different judge settings. Proposal count, number of submitted
 searches, and translation cap are distinct quantities.}
 \label{tab:app-current-evaluation-settings}
 \setlength{\tabcolsep}{4pt}
 \begin{tabular}{@{}p{0.19\linewidth}p{0.08\linewidth}p{0.45\linewidth}p{0.18\linewidth}@{}}
 \toprule
 Experiment & States & Proposal pool and selection & Translation cap \\
 \midrule
 Recovery curves & 200 & ASPIRE: up to 96 proposals, varying judge
 threshold; Random: up to 80 starts & 300 / 450 \\
 Ablations & 50 & Judge disabled for every variant; slot counts and
 objectives as in Table~\ref{tab:ablation} & 300 \\
 Wall time & 50 & ASPIRE: $\widehat p_m\ge10^{-5}$; reuse library sizes
 as in Section~\ref{sec:rev-walltime} & 300; Random 450 \\
 \bottomrule
 \end{tabular}
\end{table}

\subsection{Recorded search-force count}
\label{app:rev-protocol}
For the fixed evaluation bank $\mathcal X$, let $I_X$ index recorded
completed searches and $F_{X,m}$ count entries to
\texttt{Dimer.dimer\_force} in search $m$. The mean search-force cost is
\begin{equation}
 \bar F_{\mathrm{search}}=
 \frac1{|\mathcal X|}\sum_{X\in\mathcal X}\sum_{m\in I_X}F_{X,m}.
 \label{eq:app-force-accounting}
\end{equation}
The ledger includes recorded force calls from failed, capped, duplicate,
and accepted-unmatched searches. Each entry counts one invocation of the
search-force routine for the configuration.

The search-force ledger excludes LAMMPS initialization, native
\texttt{SPSRELAX} force work outside the routine, neural inference, and
catalogue lookup.

ASPIRE and reuse searches use a 300-translation cap; the Random baseline
uses up to 80 searches with a 450-translation cap. Preloaded searches use
the native fine-step parameter family. The reported comparisons therefore
evaluate these specified operating protocols. Each translation may involve
multiple force calls for orientation updates.

\subsection{Wall time, retrieval size, and experiment records}
Wall time is reported separately from search-force counts. In the
main-text wall-time comparison, Random and event-reuse baselines use
16 CPUs, while ASPIRE uses 16 CPUs and one A100 GPU. This comparison
uses the shared 50-state subset, with ASPIRE judge threshold
$\tau=10^{-5}$ as stated in Figure~\ref{fig:rev-walltime}.
The timing results complement search-force counts with recorded elapsed
time under these hardware settings.

KDB-494 contains 494 stored templates with 988 directed query entries;
KeyBank-32/128 designate libraries from 32/128 training configurations.
Retained proposals undergo native Dimer refinement. KeyBank's truncation flag
identifies retrievals stopped at the mapping limit. Detailed baseline
construction appears in Appendices~\ref{app:rev-kdb} and
\ref{app:rev-keybank}.

The experiment populations comprise the 200-configuration recovery test,
the shared 50-configuration ablation and wall-time subset of
Appendix~\ref{app:rev-eval50}, the judge's 200 training/50 monitoring
configurations, and the 33 selected successful MDS trajectories. Each result
is interpreted with its state selection, checkpoints, thresholds, library,
Dimer settings, scorer, and force/timing protocol. The MDS sample provides
a complementary illustration of search geometry.

\stopcontents[appendixoutlinecontinued]
\end{document}